\documentclass[12pt]{article}
\usepackage{setspace}
\usepackage{amsfonts}
\usepackage{srcltx}
\usepackage{url}
\usepackage{multirow}
\usepackage{tabularx}
\usepackage{array,makecell}
\usepackage{rotating}
\usepackage[caption = true]{subfig}
\usepackage{graphicx}
\usepackage{epstopdf}
\usepackage{caption}
\usepackage{amsmath}
\usepackage{amsthm}
\usepackage{amssymb}
\usepackage{titlesec}
\usepackage[authoryear]{natbib}
\usepackage{microtype}
\usepackage{enumerate}
\usepackage{hyperref}
\usepackage{xcolor}
\usepackage{here}
\usepackage{authblk}
\makeatletter
\numberwithin{equation}{section}
\newtheorem{theorem}{Theorem}[section]
\newtheorem{lemma}{Lemma}[section]
\newtheorem{proposition}[theorem]{Proposition}

\theoremstyle{definition}
\newtheorem{definition}{Definition}[section]

\newtheorem{example}{Example}[section]
\theoremstyle{remark}

\newtheorem{assumption}{Assumption}[section]

\usepackage{booktabs}
\usepackage{threeparttable} 

\makeatletter

\newcommand{\Rmnum}[1]{\expandafter\@slowromancap\romannumeral #1@}
\makeatother

\newcommand{\eins}{\boldsymbol{1}}
\newcommand{\argmin}{\operatornamewithlimits{arg \, min}}

\newcommand{\sign}{\text{sign}}

\usepackage{mathptmx}
\usepackage[scaled=.90]{helvet}
\usepackage{courier}

\newcommand{\blind}{1}

\begin{document}
\def\spacingset#1{\renewcommand{\baselinestretch}%
{#1}\small\normalsize} \spacingset{1}


\if1\blind
{
  \title{\bf Best-scored Random Forest Classification}
  \author{Hanyuan Hang$^*$, 
  	Xiaoyu Liu$^*$, 
  	and Ingo Steinwart$^{**}$
  	\\
    ${}^*$Institute of Statistics and Big Data,
    Renmin University of China 
    \and
    ${}^{**}$Institute of Stochastics and Applications,
    University of Stuttgart
}
  \maketitle
} \fi

\if0\blind
{
  \bigskip
  \bigskip
  \bigskip
  \begin{center}
    {\LARGE\bf xxx}
\end{center}
  \medskip
} \fi

\bigskip
\begin{abstract}
We propose an algorithm named best-scored random forest for binary classification problems. The terminology \emph{best-scored} means to
select the one with the best empirical performance out of a certain number of purely random tree candidates as each single tree in the forest. In this way,
the resulting forest can be more accurate than the original purely random forest. 
From the theoretical perspective, 
within the framework of regularized empirical risk minimization penalized on the number of splits,
we establish almost optimal convergence rates for the proposed best-scored random trees under certain conditions which can be extended to the best-scored random forest. 
In addition, we present a counterexample to illustrate that in order to ensure the consistency of the forest, every dimension must have the chance to be split.
In the numerical experiments, 
for the sake of efficiency, we employ an adaptive random splitting criterion. 
Comparative experiments with other state-of-art classification methods demonstrate the accuracy of our best-scored random forest.
\end{abstract}

\noindent%
\vfill
\newpage
\spacingset{1.45}

\section{Introduction} \label{sec::introduction}
Ensemble methods
are machine learning methods where each learner provides an estimate of the target variables and all estimates are then combined in some fashion hopefully to reduce the generalization error compared to a single learner. Random forest, devised by Breiman in the early 2000s \citep{breiman2001random}, is on top of the list of the most successful ensemble methods currently applied to deal with a wide range of prediction problems. 
With the help of combining several randomized decision trees during the training phase and aggregating their predictions by averaging or voting, this supervised learning procedure has shown high-quality performance in settings where the number of variables involved is much larger than the number of observations. 
Moreover, the applications of various versions of random forest in a large number of fields including bioinformatics \citep{D06}, survival analysis \citep{Ishwaran08}, 3D face analysis \citep{Fanelli13}, cancer detection \citep{Paul18}, stereo matching \citep{Park18,Kim17}, head and body orientation \citep{Lee17}, head and neck CT images for radiotherapy planning \citep{Wang18}, 3D human shape tracking \citep{Huang17}, gaze redirection problem in images \citep{Kononenko17}, salient object detection and segmentation \citep{Song17}, organ segmentation \citep{Farag17}, visual attribute prediction \citep{Li16}, scene labeling \citep{Cordts17} further demonstrate the practical effectiveness of the algorithm, where splitting criterions utilizing sample information such as those based on information gain \citep{Quinlan86}, information gain ratio \citep{Quinlan93} and Gini dimension \citep{Breiman84} are employed. However, in general,
from the statistical perspective,
those variants of random forest classifiers are not universal consistent. 
For example,
\cite{Biau08} constructs a specific distribution as a counterexample where the classifier does not converge due to the nature of Gini dimension criterion. 

Because of its wide applications, efforts have been paid in the random forest society to further investigate various versions of random forests from the theoretical perspective. For instance, \cite{ArGe14} studies the approximation error of some purely random forest models and shows the error of an infinite forest decreases at a faster rate (with respect to the size of each tree) than a single tree. Furthermore, concerning with classification problems,
\cite{Biau08} established weak convergence of 
the purely random forest, which is
the radically simplest version of random forest classifiers \citep{Breiman00}, while
\cite{Biau12} shows that Breiman's procedure is consistent and adapts to sparsity, in the sense that the rate of convergence depends only on the number of strong/active features, and
\cite{Scornet15} proves the $L_2$ consistency for Breiman's original algorithm in the context of additive models.
Moreover,
\cite{WaWa14} builds an adaptive concentration as a framework for describing the statistical properties of adaptively grown trees by viewing training trees as occurring in two stages. 
Finally, 
\cite{MoGa17} establishes the consistency of modified Mondrian Forests \citep{LaRo14} that can be implemented online while achieving the minimax rate for the estimation of Lipschitz continuous functions.

In this paper, we propose a random forest algorithm named \textit{best-scored random forest} which not only achieves almost optimal convergence rates, but also 
enjoys satisfactory performance on several benchmark data sets.
The main contributions of this paper are twofold: 
\emph{(i)} Concerning with theoretical analysis, we establish almost optimal convergence rates for the best-scored random trees under certain restrictions and successfully extend it to the case of best-scored random forest. The convergence analysis is conducted under the framework of regularized empirical risk minimization. 
As is in the traditional learning theory analysis where decomposing the error term into bias and variance terms is attached great significance, we proceed the analysis by decomposing the error term into data-free and data-dependent error terms, respectively. To be more precise, the data-free error term can be dealt with by applying techniques from the approximation theory whereas the data-dependent error term can be resolved through exploiting arguments from the empirical process theory, such as oracle inequalities with regard to regularized empirical risk minimization in our case. 
	In order to have a more rigorous understanding on the consistency of random forest, we present a counterexample which explicitly demonstrate that all dimensions should have the probability to be split from in order to achieve the consistency.
\emph{(ii)} 
	When it comes to numerical experiments, 
	we first improve the original random splitting criterion by proposing an adaptive random partition method, which
    differs from the purely random partition in its conducting the node selection process in another way. Specifically, at each step, we need to randomly select a sample point from the training data set and the to-be-split node is the one which this point falls in. This idea follows the fact that when randomly picking sample points from the whole training data set, nodes with more samples will be more likely to be selected while nodes with fewer samples are less possible to be chosen. Consequently, we have a greater probability to obtain cells with sample sizes evenly distributed.
	Empirical experiments further show that the adaptive/recursive method enhances the efficiency of the algorithm for it actually increases the effective number of splits compared to the original purely random partition method. 

The rest of this paper is organized as follows: Section \ref{sec::preliminaries} introduces some fundamental notations and definitions related to the best-scored random forest. We provide our main results and statements on the oracle inequalities and learning rates of the best-scored random forest in Section \ref{sec::MainResults}. Some comments and discussions on the main results will be also presented in this section. The main analysis on bounding error terms is provided in Section \ref{sec::ErrorAnalysis}. 
A counterexample aiming at giving a more rigorous understanding of the consistency of random forest is presented in Section \ref{sec::Counterexample}.
Numerical experiments conducted upon comparisons between best-scored random forest and other classification methods are given in Section \ref{sec::experiments}. All the proofs of Section \ref{sec::MainResults} and Section \ref{sec::ErrorAnalysis} can be found in Section \ref{sec::proofs}. Finally, we conclude this paper with a brief discussion in Section \ref{sec::conclusion}.

\section{Preliminaries} \label{sec::preliminaries}

\subsection{Notations} \label{subsec::notations}
Throughout this paper, we suppose that the data set $D_n = \{(x_1, y_1), \ldots, (x_n, y_n)\}$ given is of independent and identically distributed $\mathbb{R}^d \times \{-1, 1\}$-valued random variables with the same distribution as the generic pair $(X, Y)$, where $X$ is the feature vector while $Y$ is the binary label. The joint distribution $\mathrm{P}$ of $(X, Y)$ is determined by the marginal distribution $\mathrm{P}_X$ of $X$ and the \textit{a posteriori} probability $\eta : \mathbb{R}^d \to [0, 1]$ defined by
\begin{align} \label{etaFunction}
\eta(x) : = \mathrm{P}(Y = 1 | X = x). 
\end{align}
The learning goal of binary classification is to find a decision function $f : \mathbb{R}^d \to \{-1, 1\}$ such that for the new data $(x, y)$, we have $f(x) = y$ with high probability.

In order to precisely describe our learning goal, also for the ease of convenience, we suppose that $x \in \mathcal{X} := [0, 1]^d$ and $Y \in \mathcal{Y} : = \{ -1, 1 \}$. Under that condition, it is legitimate to consider the classification loss $L = L_{\text{class}} : \mathcal{Y} \times \{ -1, 1 \} \to \{ 0 , 1 \}$ defined by 
$L(y, t) := \eins_{\{y \neq t\}}$. 
We define the risk of the decision function $f$ by
\begin{align*}
\mathcal{R}_{L, \mathrm{P}}(f) : = \int_{\mathcal{X} \times \mathcal{Y}} L(y, f(x)) \, d\mathrm{P}(x, y) 
= \mathrm{P}(f(X) \neq Y),
\end{align*}
and the empirical risk by
\begin{align*}
\mathcal{R}_{L, D}(f) : = \frac{1}{n} \sum_{i=1}^n L(y_i, f(x_i)),
\end{align*}
where $D : = \frac{1}{n} \sum_{i=1}^n \delta_{(x_i, y_i)}$ denotes the average of Dirac measures at $(x_i, y_i)$. The smallest possible risk
$\mathcal{R}_{L, \mathrm{P}}^* : = \inf \{ \mathcal{R}_{L, \mathrm{P}}(f) | f : \mathcal{X} \to \mathcal{Y} \}$
is called the Bayes risk, and a measurable function $f^*_{L, \mathrm{P}} : \mathcal{X} \to \mathcal{Y}$ so that $\mathcal{R}_{L, \mathrm{P}}(f^*_{L, \mathrm{P}}) = \mathcal{R}_{L, \mathrm{P}}^*$ holds is called Bayes decision function. By simple calculation, we obtain the Bayes decision function
$
f^{*}_{L,\mathrm{P}} = \text{sign}(2\eta - 1).
$

\subsection{Purely Random Tree and Forest} \label{PRTF}


Considering how the purely random forest put forward by \cite{Breiman00} plays a vital and fundamental role in theory and practice, we base our following analysis on this specific kind of random forest. 
To illustrate 
one possible general construction process of one tree in the forest, some randomizing variables are in need to describe the selection process of the node, the coordinate and the cut point at each step.
Therefore, we introduce a random vector $Q_i : = (L_i, R_i, S_i)$ which describes the splitting rule at the $i$-th step of tree construction. 
To be specific, $L_i$ denotes the randomly chosen leaf to be split at the $i$-th step from the candidates which are defined to be all the leaves presented in the $(i-1)$-step, thus the leaf choosing procedure follows a recursive manner.
The random variable $R_i \in \{ 1, \ldots, d\}$ denotes the dimension chosen to be split from for the $L_i$ leaf. The random variables $\{R_i$, $i \in \mathbb{N}\}$  are independent and identically multinomial distributed with each dimension having equal probability to be chosen.
The random variables $\{S_i$, $i \in \mathbb{N}\}$ are independent and identically distributed from $\mathrm{Unif}(0, 1)$. These variables are recognized as proportional factors representing the ratio between the length of the newly generated leaf in the $R_i$-th dimension after completing the $i$-th split and the length of the being cut leaf $L_i$ in the $R_i$-th dimension. In other words, the length of the newly generated leaf in the $R_i$-th dimension can be calculated by multiplying length of the leaf $L_i$ in the $R_i$-th dimension and the proportional factor $S_i$.

\begin{figure*}[htbp]
	\begin{minipage}[t]{0.8\textwidth}  
		\centering  
		\includegraphics[width=\textwidth]{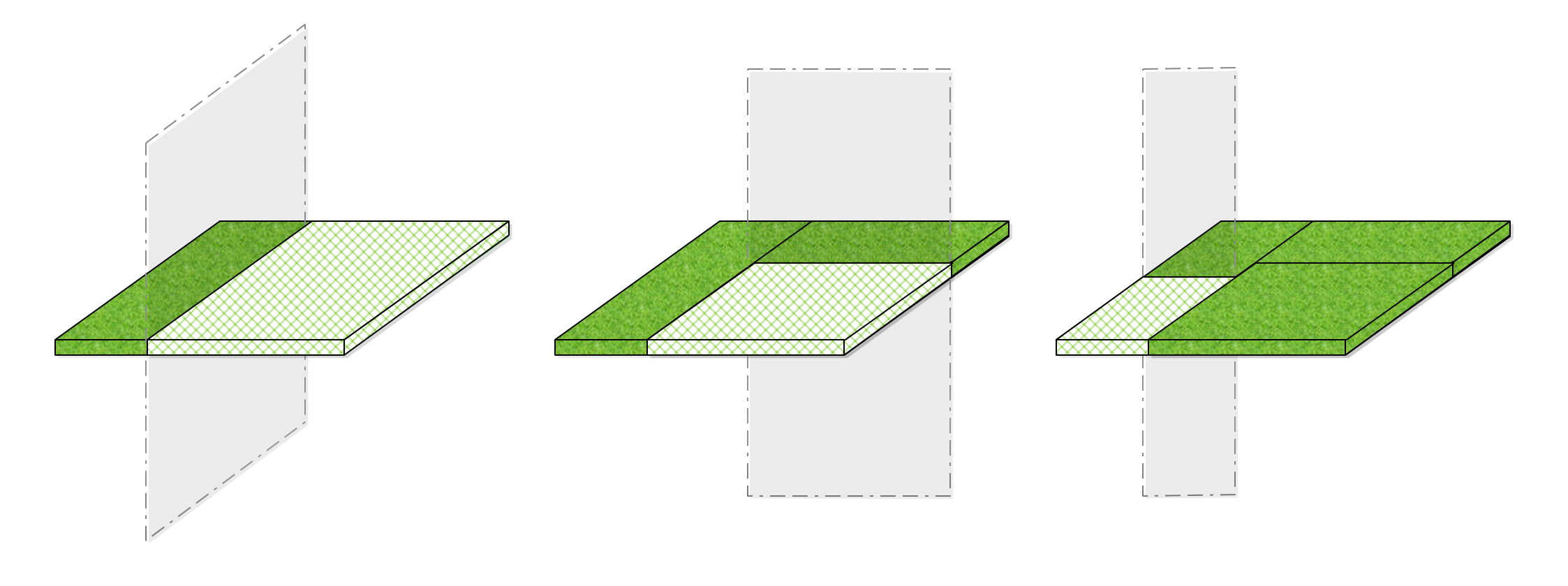}  
	\end{minipage}  
	\centering  
	\caption{Possible construction procedures of $3$-split axis-parallel purely random partitions in a $2$-dimensional space. }
	\label{fig:ap}
\end{figure*}

Statements mentioned above give a quantifiable description of the splitting process of the purely random decision tree. However, 
for a clearer understanding of this splitting approach, 
we might give one simple example where we develop a partition on $A = [0, 1]^d$.
To begin with, we randomly choose a dimension out of $d$ candidates and uniformly split at random from the chosen dimension so that $A$ is split into two leaves which are $A_{1, 1}$ and $A_{1, 2}$, respectively. Secondly, a leaf is chosen uniformly at random, say $A_{1, 1}$, and we once again randomly pick the dimension and do the split on that leaf, which leads to a partition consisting of $A_{2, 1}, A_{2,2}, A_{1,2}$. Thirdly, a leaf is randomly selected from all three leaves, say $A_{2,2}$,  and the third split is once again conducted on that chosen leaf with dimension and node chosen the same way as before, which leads to a partition consisting of $A_{2,1}, A_{3,1}, A_{3,2}, A_{1,2}$. The construction process 
will continue in this approach until the number of splits reaches our satisfaction. To notify, the above process can be extended to the general cases where construction is conducted on the feature space $\mathcal{X}$.
All of the above procedures lead to a so-called partition variable $Z$ which is defined by 
$Z:= (Q_0, Q_1, \ldots)$
and takes values in space $\mathcal{Z}$.
The probability measure of $Z$ is denoted as $\mathrm{P}_Z$.

Assume that any specific partition variable $Z \in \mathcal{Z}$  can be treated as a latent splitting policy. Considering a partition $Z$ with $p$ splits, we denote $\mathcal{A}_{Z, p} : = \mathcal{A}_{(Q_0, \ldots, Q_p)}$
and define the resulting collection of cells as $\mathcal{A}_{Z, p}$ which is a partition of $\mathcal{X}$, where $\mathcal{A}_{Z, 0} : = \mathcal{X}$. If we focus on certain sample point $x \in \mathcal{X}$, then the corresponding cell in which that point falls is defined by $A_{Z, p}(x)$. Here, we introduce the random tree decision rule,
that is a map $g_{Z, p} : \mathcal{X} \to \{-1, 1\}$ defined by
\begin{align} \label{TreeRule}
	g_{Z, p} (x) = \sign(h_{Z, p} (x)).
\end{align}
where
\begin{align*}
h_{Z, p} (x) : = \sum_{i=1}^n Y_i 
\eins_{\{ x_i \in A_{Z, p} (x) \}}.
\end{align*}

Assume that our random forest is determined by the latent splitting policies $Z = \{Z_1, \ldots, Z_m\}$ consisting of independent and identically distributed random variables drawn from $\mathcal{Z}$, and the number of splits for $m$ trees are presented as $p = (p_1, \ldots, p_m)$. As usual, the random forest decision rule can be defined by
\begin{align*}
f_{Z, p} (x) : = 
\begin{cases}
1 & \text{ if } \sum\limits_{t=1}^m g_{Z_t, p_t} (x) > 0 
\\
-1 & \text{ otherwise}.
\end{cases}
\end{align*}

\subsection{Best-scored Random Forest}

Considering that the preliminary work of analysis of a random forest is to focus on how to give appropriate partitions to several independent trees, we first define a function set which contains all the possible partitions as follows:
\begin{align} \label{Space}
\mathcal{T} 
:= \biggl\{ \sum_{j=0}^p c_j \boldsymbol{1}_{A_j} : p \in \mathbb{N}, c_j \in \{-1, 1\}, \bigcup_{j=0}^p A_j = \mathcal{X},  A_s \cap A_{\tilde{s}} = \emptyset, s \neq \tilde{s} \biggr\}.
\end{align}
In this paper, without loss of generality, we only consider cells with the shape of
$A_j = \bigotimes_{i=1}^d [a_{ij}, b_{ij}]$.
To be specific, choosing $p \in \mathbb{N}$ as the number of splits, the resulting leaves presented as $A_0, A_1, \ldots, A_p$ in fact construct a partition of $\mathcal{X}$ with $p$ splits. What we also attach great significance is that $c_j$ is the value of leaf $A_j$, and thus the set $\mathcal{T}$ contains all the potential decision rules.
Moreover, for fixed $p \in \mathbb{N}$,
we denote the collection of trees with number of splits $p$ as
\begin{align} \label{Tq}
\mathcal{T}_p : = \bigg\{\sum_{j=0}^p c_j \eins_{A_j} : c_j \in \{-1, 1\}, \bigcup_{j=0}^p A_j = \mathcal{X},
A_s \cap A_{\tilde{s}} = \emptyset, s \neq \tilde{s} \bigg\},
\end{align}
where we should emphasize that all trees in \eqref{Tq} must follow our specific construction procedure described in Section \ref{PRTF}.
It can be easily verified that the following nested relation holds:
\begin{align} \label{NestRelation}
\mathcal{T}_{q} \subset \mathcal{T}_p
\quad
\text{ for } \ q \leq p.
\end{align}

Let $k$ be a fixed number to be chosen later and assume that the forest consists of $m$ trees. 
For $t \in \{ 1, \ldots, m \}$, suppose that $Z_{1t}, \cdots, Z_{kt}$ are independent and identically distributed random variables drawn from $\mathcal{Z}$, which are also the splitting policies of those trees. For $\ell \in \{1, \ldots, k\}$, we might as well derive a random function set induced by that specific splitting policy $Z_{\ell t}$ after $p$ steps as
\begin{align} \label{SpaceZl}
\mathcal{T}_{Z_{\ell t}, p} : = \bigg\{\sum_{j=0}^{p} c_{j} \boldsymbol{1}_{A_j} : c_{j} \in \{-1,1\}, A_j \in \mathcal{A}_{Z_{\ell t}, p} \bigg\},
\end{align}
which is a subset of $\mathcal{T}$. We also denote $\mathcal{T}_{Z_{\ell t}} : = \bigcup_{p \in \mathbb{N}} \mathcal{T}_{Z_{\ell t}, p}$.

Having found an appropriate random tree decision rule under policy $Z_{\ell t}$ denoted as $g_{Z_{\ell t}}$, we are supposed to scrutinize the convergence properties of that rule. To this end, we need to introduce the framework of regularized empirical risk minimization,
see also Definition 7.18 in \cite{StCh08}.
Let 
$\mathcal{L}_0$ be the set of measurable functions on $\mathcal{X}$,
$\mathcal{F} \subset \mathcal{L}_0 ( \mathcal{X} )$ be a non-empty set,
$L : \mathcal{X} \times \mathcal{Y} \times \mathbb{R} \to [0, \infty)$ be a loss function, 
and
$\Omega : \mathcal{F} \to [0, \infty)$ be a function. 
A learning method (see e.g. Definition 6.1 in \cite{StCh08})
whose decision function $f_D$ satisfying
\begin{align*}
\mathcal{R}_{L, D} (f_D) + \Omega (f_D) = \inf_{f \in \mathcal {F}} \mathcal{R}_{L, D} (f) + \Omega (f)
\end{align*}
for all $n \ge 1$ and $D \in (\mathcal{X} \times \mathcal{Y})^n$ is called regularized empirical risk minimization.

To notify, we put forward an idea that the number of splits $p$ is the one should be penalized on. The reason why the penalization on $p$ is necessary is that it not only significantly reduces the huge amount of calculation, which then the number of splits is bounded and the function set has a finite VC dimension, but more importantly refrains from overfitting. With the same data set $D_n$, the above regularized empirical risk minimization problem with respect to each function space $\mathcal{T}_{Z_{\ell t}}$ turns into
\begin{align} \label{OptimizationProblem}
\min_{p \in \mathbb{N}} \min_{g \in \mathcal{T}_{Z_{\ell t}, p}} \ \lambda p^2 + \mathcal{R}_{L, D} (g),\ \ \ell = 1, \ldots, k.
\end{align}

It is noteworthy that the regularized empirical risk minimization under any policy can be bounded simply by having a quick look at the situation where no split is applied to $\mathcal{X}$. As a consequence, the optimization problem can be represented as
\begin{align*}
\min_{p \in \mathbb{N}} \min_{g \in \mathcal{T}_{Z_{\ell t}, p}} \ \lambda p^2 + \mathcal{R}_{L, D} (g)
\le \mathcal{R}_{L, D} (1) \le 1,
\end{align*}
where $\mathcal{R}_{L, D} (1)$ denotes the empirical risk for taking $g(x) = 1$ for all $x \in \mathcal{X}$ with $p = 0$. Thus, it can be apparently seen that the best number of splits $p$ is upper bounded by $\lambda^{-1/2}$, which leads to a capacity reduction of the underlying function set. Therefore, the following function spaces will be all added an extra condition $p \leq \lambda^{-1/2}$.

According to the random tree decision rule \eqref{TreeRule},
our goal is to solve
the above optimization problem \eqref{OptimizationProblem}
which can be further denoted by
\begin{align*}
(g_{Z_{\ell t}}, \ p_{Z_{\ell t}})
= \argmin_{p \in \mathbb{N}} \ \argmin_{g \in \mathcal{T}_{Z_{\ell t}, p}} \ \lambda p^2 + \mathcal{R}_{L, D} (g),
\quad \quad 
\ell = 1, \ldots, k,
\end{align*}
where $p_{Z_{\ell t}}$ is the number of splits of the decision function $g_{Z_{\ell t}}$.
Its population version can be denoted by
\begin{align} \label{GLP}
(g_{Z_{\ell t}}^*,\ p_{Z_{\ell t}}^*) = \argmin_{p \in \mathbb{N}} \  \argmin_{g \in \mathcal{T}_{Z_{\ell t}, p}} \ \lambda p^2 + \mathcal{R}_{L, \mathrm{P}} (g), 
\quad \quad 
\ell = 1, \ldots, k.
\end{align}

The fact that directly aggregating all random trees at hand is not always sensible, since some of them may not be able to classify the data with proper manners. For this reason, we advocate a new method named as best-scored random forest. Every tree in the random forest is chosen from $k$ candidates and the main principle is to retain only the tree yielding the minimal regularized empirical risk, which is
\begin{align}  \label{BestScoreTree}
(g_{Z_t}, \ p_{Z_t})
= \argmin_{\ell=1, \ldots, k} \ \lambda p^2_{Z_{\ell t}} + \mathcal{R}_{L, D} (g_{Z_{\ell t}}),
\end{align}
where $p_{Z_t}$ is the number of splits of $g_{Z_t}$ and $Z_t = \{ Z_{1t}, \ldots, Z_{kt} \}$.
Apparent as it is, $g_{Z_t}$ is the regularized empirical risk minimizer with respect to the random function space 
$
\mathcal{T}_{Z_t} : = \bigcup_{\ell=1}^k \mathcal{T}_{Z_{\ell t}}.
$
In other words, $g_{Z_t}$ is the solution to the regularized empirical risk minimization problem
\begin{align*}
\min_{g \in \mathcal{T}_{Z_t}} \ \lambda p^2 (g) + \mathcal{R}_{L, D} (g)
: = \min_{\ell=1, \ldots, k} \ \min_{p \in \mathbb{N}} \ \min_{g \in \mathcal{T}_{Z_{\ell t}, p}} \ \lambda p^2 + \mathcal{R}_{L, D}(g).
\end{align*}
Similarly, $g_{Z_t}^*$ is denoted as the solution of the population version of regularized minimization problem in the space $\mathcal{T}_{Z_t}$
\begin{align} \label{BestScoreMinimizerPopulation}
( g_{Z_t}^*, \, p_{Z_t}^*) 
= \argmin_{g \in \mathcal{T}_{Z_t}} \ \lambda p^2 (g) + \mathcal{R}_{L, \mathrm{P}} (g) 
= \argmin_{\ell = 1, \ldots, k}\ \lambda p_{Z_{\ell t}}^{*2} + \mathcal{R}_{L, \mathrm{P}}(g_{Z_{\ell t}}^*).
\end{align}
Again, $p_{Z_t}^*$ is the corresponding number of splits of 
$g_{Z_t}^*$.

\section{Main Results and Statements}\label{sec::MainResults}
In this section, we present main results on the oracle inequalities and learning rates for the best-scored random tree and forest. More precisely, section \ref{FundamentalAssumption} gives the fundamental assumptions for the analysis of our classification algorithm. Section \ref{sec::OracleInequality} is devoted to the oracle inequality of the best-scored random trees. Then, in section \ref{LearningRates}, we use the established oracle inequalities to derive learning rates. On account of the results of those base classifiers, learning rates of the ensemble forest will be established in section \ref{sec::LearningRatesEnsemble}. Finally, we present some comments and discussions concerning the obtained main results.

\subsection{Fundamental Assumptions} \label{FundamentalAssumption}
To present our main results, we need to make assumptions on 
the behavior of $\mathrm{P}$ in the vicinity of the decision boundary
by means of the posterior probability $\eta$ defined as in \eqref{etaFunction}.
To this end, we write
\begin{align} \label{DistanceSet}
\mathcal{X}_{1} & =  \{ x \in \mathcal{X} : \eta(x) > 1/2 \},
\nonumber \\
\mathcal{X}_{0} & = \{ x \in \mathcal{X} : \eta(x) = 1/2 \},
\nonumber \\
\mathcal{X}_{-1} & = \{ x \in \mathcal{X} : \eta(x) < 1/2\},
\end{align}
and define the distance to the decision boundary by
\begin{align} \label{DistanceBound}
\Delta(x) =
\begin{cases}
d(x, \mathcal{X}_1) & \text{ if } x \in \mathcal{X}_{-1},
\\
d(x, \mathcal{X}_{-1}) & \text{ if } x \in \mathcal{X}_1,
\\
0 & \text{ otherwise},
\end{cases}
\end{align}
where $d(x, A) = \inf_{x' \in A} d(x, x')$,
see also Definition 8.5 in \cite{StCh08}.
In the following, 
the distance $\|\cdot\|$ always denotes the $L_1$-norm
if not mentioned otherwise.

\begin{assumption} \label{NoiseExponent}
	The distribution  $\mathrm{P}$ on $\mathcal{X} \times \mathcal{Y}$ is said to have \emph{noise exponent} $\alpha \in [0, \infty]$ if there exists a constant $c_\alpha > 0$ such that
	\begin{align*}
	\mathrm{P}_X \bigl( \{ x \in \mathcal{X} : |2 \eta(x) - 1| \leq h \} \bigr) 
	\leq c_\alpha h^\alpha,
	\quad
	h \geq 0.
	\end{align*}
\end{assumption}
The notion of the noise exponent can be traced back to \cite{Tsybakov04}. 
Tsybakov's noise assumption is extensively used in the literature of dynamical systems \citep{Hang17}, Neyman-Pearson classification \citep{Zhao16}, active learning \citep{Hanneke15}, bipartite ranking \citep{Agarwal14}, etc. Assumption \ref{NoiseExponent} is intrinsically related to the analysis of the estimation error. Note that for any $x\in \mathcal{X}$, if $\eta(x)$ is close to $1/2$, then the amount of noise will be large in the labeling process at $x$. From that perspective, this assumption gives a measurement on the size of the set of points whose noise is high during the labeling process. As is widely acknowledged, points whose \emph{a posteriori} probability $\eta(x)$ are far away from $1/2$ are in favor since they provide a clear choice of labels. Consequently, the Assumption \ref{NoiseExponent} is used to guarantee that 
the probability that points with high noise occur is low. 
In that case, samples which are less useful in classification will be lesser and thus leads to a better analysis of the data-based error, that is the estimation error.

\begin{assumption} \label{MarginNoiseExponent}
	The distribution  $\mathrm{P}$ on $\mathcal{X} \times \mathcal{Y}$ has \emph{margin-noise exponent} $\beta \in [0, \infty)$ if there exists a constant $c_\beta > 0$ such that
	\begin{align*}
	\int_{\Delta(x) \le h}|2 \eta(x) - 1| \ d \mathrm{P}_X (x) \le c_\beta h^\beta,
	\quad
	h\geq 0.
	\end{align*}
\end{assumption}
The margin-noise exponent was put forward by \cite{Steinwart07} where the relationship between margin-noise exponent and noise exponent is analyzed as well. Assumption \ref{MarginNoiseExponent} measures the size of the set of points, denoted as $\{x\in \mathcal{X} : \Delta(x) \le h\}$, which are close to the opposite class and the integral is with respect to the measure $\|2 \eta - 1\| \, d \mathrm{P}_X$. This geometric noise assumption is essentially related to the approximation error, since in the context of random tree construction, only those cells which intersect the decision boundary contribute to the approximation error. Therefore, the concentration of mass near the decision boundary determines the approximation ability to some extent. 

The last assumption describes the relation between the distance to the decision boundary and the discrepancy of the posterior probability to level $1/2$, see also Definition 8.16 in \cite{StCh08}.
\begin{assumption}\label{ass::DistanceControlsNoise}
	Let $(\mathcal{X}, d)$ be a metric space, $\mathrm{P}$ be a distribution on $\mathcal{X} \times \mathcal{Y}$ , and $\eta : \mathcal{X} \to [0, 1]$ be a version of its posterior probability. We say that the associated \emph{distance to the decision boundary $\Delta$ defined by \eqref{DistanceSet} controls the noise by the exponent} $\gamma\in [0, \infty)$ if there exists a constant $c > 0$ such that
	\begin{align*}
	|2\eta(x)-1| \leq c\Delta(x)^{\gamma}
	\end{align*}
	holds for $\mathrm{P}_X$-almost all $x \in \mathcal{X}$.
\end{assumption}

Note that since $|2\eta(x) - 1| \leq 1$ for all $x \in \mathcal{X}$, the above assumption becomes trivial whenever $\Delta(x) \geq c^{-1/\gamma}$ . Consequently, the assumption only considers points  $x \in \mathcal{X}$ with sufficiently small distance to the opposite class. In short, it states that $\eta(x)$ is close to the level $1/2$ of \textit{complete noise} if $x$ approaches the decision boundary.

\subsection{Oracle Inequality for Best-scored Random Tree} \label{sec::OracleInequality}

We now establish an oracle inequality for the best-scored random tree. 

\begin{theorem} [Oracle inequality for best-scored random trees] \label{OracleInequality}
	Let	$L$ be the classification loss,
	$\alpha$ be the noise exponent as in Assumption \ref{NoiseExponent}, 
	$\vartheta = \alpha / (1 + \alpha)$.
	Then, for all fixed $k \in \mathbb{N}$, $\tau > 0$, $\lambda > 0$ and $\delta \in (0, 1)$, conditioned on any $Z_t = (Z_{1,t}, \ldots, Z_{k,t}) \in \mathcal{Z}^k$,
	the $t$-th best-scored random trees \eqref{BestScoreTree} 
	using $L$ satisfies
	\begin{align*}
	\lambda p_{Z_t}^2 + \mathcal{R}_{L, \mathrm{P}}(g_{Z_t}) - \mathcal{R}_{L, \mathrm{P}}^*
	& \leq 9 \bigl( \lambda p_{Z_t}^{*2} + \mathcal{R}_{L, \mathrm{P}}(g_{Z_t}^*) - \mathcal{R}_{L, \mathrm{P}}^* \bigr)
	\\
	& \phantom{=}
	+ C \biggl( \frac{1}{\lambda n^2} \biggr)^{\frac{1}{3 - 2 (1 - \delta) \vartheta}}
	+ 3 \biggl( \frac{72 V \tau}{n} \biggr)^{\frac{1}{2 - \vartheta}} 
	+ \frac{15 \tau}{n}
	\end{align*}
	with probability $\mathrm{P}_{| Z_{1,t}, \ldots, Z_{k,t}}$ at least $1 - 3 e^{- \tau}$, where 
	$V$ is a constant depending on $\alpha$
	and $C$ is a constant depending on $d$, $\delta$ and $\vartheta$
	which will be specified later in the proof.
\end{theorem}

\subsection{Learning Rates for Best-scored Random Trees} \label{LearningRates}
Based on the established oracle inequality,
we now state our main results on the convergence rates for best-scored random trees.

\begin{theorem} \label{ConvergenceRates}
	Let	$L$ be the classification loss,
	$\alpha$ be the noise exponent as in Assumption \ref{NoiseExponent}, 
	$\vartheta = \alpha / (1 + \alpha)$,
	$\beta$ be the margin-noise exponent as in Assumption \ref{MarginNoiseExponent}, $k$ be the number of candidate trees.
	Then, for all $\tau > 0$, all $\delta \in (0, 1)$, and all $n > 1$,
	with probability $\mathrm{P} \otimes \mathrm{P}_Z$ at least $1 - 4 e^{- \tau}$, the $t$-th tree in the best-scored random forest learns with $p = n^{- \frac{4d}{c_T \beta (2 - (1-\delta)\vartheta) + 4d}}$,
    the rate
	\begin{align} \label{OOptimalLearningRate}
	\mathcal{R}_{L,\mathrm{P}}(g_{Z_t}) 
	- \mathcal{R}_{L, \mathrm{P}}^*
	\leq C  n^{-\frac{c_T\beta}{c_T\beta(2- (1 - \delta) \vartheta)+4d} },
	\end{align}
	where $c_T = 0.22$, $C$ is a constant which will be specified in the proof later,
	independent of $n$ and only depending on constants $k$, $d$, $\vartheta$, $\beta$, $\delta$, and $\tau$.
\end{theorem}

Let us briefly discuss the constants $c_T$ and $C$.
If we consider the deterministic binary tree, then for tree with number of splits $p$, the effective number of splits for each dimension is approximately $\log p$. However, we need to take randomness into consideration, which leads to the decrease of the effective number of splits written as $c_T \log p$ with $c_T = 0.22$. 
From the proof of this theorem later, we see that
the constant $C$ deceases to some constant when the number of candidate trees $k$ increases.
Moreover, note that the constant $\delta$ can be taken as small as possible.
Therefore, under Assumptions \ref{NoiseExponent} and \ref{MarginNoiseExponent},
the learning rates \eqref{OOptimalLearningRate} are close to
\begin{align*}
\mathcal{O} \Bigl( n^{-\frac{c_T\beta}{c_T\beta(2-\vartheta)+4d} } \Bigr),
\end{align*}
which is optimal only when both $\alpha$ and $\beta$ converge to infinity simultaneously.
However, as the following example shows, 
if $\eta$ belongs to
certain H\"{o}lder spaces $C^{\gamma}$, $\alpha$ and $\beta$ cannot converge to infinity simultaneously.
Note that if $\rho_X$ is the Lebesgue measure on $\mathcal{X}$ and
the Bayes boundary $\partial\Omega^*$ has nonzero $(d - 1)$-dimensional Hausdorff measure,
then \cite{Binev14} shows that the constraint $\alpha\gamma \leq 1$ must hold.


\begin{example}
	Lemma A.2 in \cite{BlSt18} shows that if $\eta$ is H\"{o}lder-continuous with exponent $\gamma \in [0,1]$, then $\Delta$ controls the noise from above with exponent $\gamma$,  i.e., there exists a constant $c > 0$ such that $|2\eta(x)-1| \leq c \ \Delta(x)^{\gamma}$ holds
	for $P_X$-almost all $x \in \mathcal{X}$.
	This together with Lemma 8.23 in \cite{StCh08} implies that $\mathrm{P}$ has margin exponent $q = \gamma \alpha$ and margin-noise exponent $\beta = \gamma(\alpha + 1)$.
	Consider the Lipschitz continuous space where
	$\gamma = 1$, then the restriction $\gamma \alpha\leq 1$ implies that $\alpha \leq 1$ and consequently
	the convergence rate becomes 
	$$
	\mathcal{O} \Bigl( n^{-\frac{c_T(\alpha + 1)}{c_T(\alpha + 1)(2-\vartheta)+4d} } \Bigr)
	\geq  \mathcal{O} \Bigl( n^{-\frac{0.11}{0.11 + d} } \Bigr),
	$$
	which is obviously slower than the minimax rate $\mathcal{O}(n^{-\frac{2}{2+d}})$.
\end{example}

Nevertheless, under the Assumptions \ref{ass::DistanceControlsNoise} and \ref{NoiseExponent},
the following theorem establishes asymptotically optimal convergence rates.

\begin{theorem} \label{the::OptimalConvergenceRates}
	Let	$L$ be the classification loss,
	$\alpha$ be the noise exponent as in Assumption \ref{NoiseExponent}, 
	$\vartheta = \alpha / (1 + \alpha)$,
	the associated distance to the decision boundary $\Delta$ defined by \eqref{DistanceSet} controls the noise by the exponent $\gamma\in [0, \infty)$ as in Assumption \ref{ass::DistanceControlsNoise}, $k$ be the number of candidate trees.
	Then, for $\delta \in (0, 1)$ and all $n > 1$,
	with probability $\mathrm{P} \otimes \mathrm{P}_Z$ at least $1 - 4 e^{- \tau}$, the $t$-th tree in the best-scored random forest  learns with $p = n^{- \frac{4d}{c_T \gamma(\alpha + 1) (2 - (1-\delta)\vartheta) + 4d}}$,
	the rate
	\begin{align} \label{OptimalLearningRateOptimal}
	\mathcal{R}_{L,\mathrm{P}}(g_{Z_t}) 
	- \mathcal{R}_{L, \mathrm{P}}^*
	\leq C  n^{-\frac{c_T\gamma(\alpha + 1)}{c_T\gamma(\alpha + 1)(2- (1 - \delta) \vartheta)+4d} },
	\end{align}
	where $c_T = 0.22$ and $C$ is a constant which will be specified in the proof later,
	independent of $n$ and only depending on constants $k$, $d$, $\vartheta$, $\delta$ and $\tau$.
\end{theorem}

Note that the constant $\delta$ can be taken as small as possible.
Therefore, if the noise exponent $\alpha$ (and thus $\vartheta$) is sufficiently large, 
the learning rate \eqref{OptimalLearningRateOptimal} is close to
$\mathcal{O}(n^{-1})$. In other words, we achieve asymptotically the optimal rate, see e.g. \cite{Tsybakov04,Steinwart07}. 
Moreover, if $\gamma$ or $\alpha$ is large, the rate is rather insensitive to the input dimension $d$.

\subsection{Learning Rates for Ensemble Forest} \label{sec::LearningRatesEnsemble}
Inspired by the classical random forest, we propose a diverse and thus more accurate version of random forest.  
Here, we intend to construct our best-scored random forest basing on the majority voting result of $m$ best-scored trees, each of which is generated according to the procedure in \eqref{BestScoreTree}.

Let $g_{Z_t},1 \le t \le m$, $Z_t = \{Z_{1t}, \ldots, Z_{kt}\}$ be the best-scored classification trees determined by the criterion mentioned above. As usual, we perform majority voting to make the final decision
\begin{align*}
f_Z =
\begin{cases}
1 & \text{if} \ v = v_{+} - v_{-} = \sum_{t=1}^m g_{Z_t} \ge 0
\\
-1 & \text{otherwise},
\end{cases}
\end{align*}
where 
\begin{align} \label{vPlusvMinus}
v_{+} = \sum_{t=1}^m \eins_{\{g_{Z_t} = 1\}}
\quad
\text{ and } 
\quad
v_{-} = \sum_{t=1}^m \eins_{\{g_{Z_t} = -1\}}.
\end{align}

\begin{theorem} \label{ConvergenceRateForest}
	Let	$L$ be the classification loss,
	$\alpha$ be the noise exponent as in Assumption \ref{NoiseExponent}, 
	$\vartheta = \alpha / (1 + \alpha)$,
	$\beta$ be the margin-noise exponent as in Assumption \ref{MarginNoiseExponent},
	$m$ be the number of best-scored trees in the forest and
	$k$ be the number of candidate trees. Then, for $\delta \in (0,1)$ and all $n > 1$,
	with probability $\mathrm{P} \otimes \mathrm{P}_Z$ at least $1 - 4 e^{-\tau}$, 
	learns with $p = n^{- \frac{4d}{c_T \beta (2 - (1-\delta)\vartheta) + 4d}}$,
	the rate
	\begin{align*}
	\mathcal{R}_{L,\mathrm{P}}(f_Z) 
	- \mathcal{R}_{L, \mathrm{P}}^*
	\leq C n^{-\frac{c_T\beta}{c_T\beta(2- (1 - \delta) \vartheta)+4d}},
	\end{align*}
	where $c_T = 0.22$ and $C$ is a constant which will be specified in the proof later,
	independent of $n$ and only depending on $m$, $k$, $d$, $\vartheta$, $\beta$, $\delta$ and $\tau$.
\end{theorem}

\begin{theorem} \label{the::ForestOptimalConvergenceRates}
	Let	$L$ be the classification loss,
	$\alpha$ be the noise exponent as in Assumption \ref{NoiseExponent}, 
	$\vartheta = \alpha / (1 + \alpha)$,
	the associated distance to the decision boundary $\Delta$ defined by \eqref{DistanceSet} controls the noise by the exponent $\gamma\in [0, \infty)$ as in Assumption \ref{ass::DistanceControlsNoise}, $m$ be the number of best-scored trees in the forest and $k$ be the number of candidate trees.
	Then, for $\delta \in (0, 1)$ and all $n > 1$,
	with probability $\mathrm{P} \otimes \mathrm{P}_Z$ at least $1 - 4 e^{- \tau}$, the best-scored random forest learns with $p = n^{- \frac{4d}{c_T \gamma(\alpha + 1) (2 - (1-\delta)\vartheta) + 4d}}$,
	the rate
	\begin{align} \label{OptimalLearningRate}
	\mathcal{R}_{L,\mathrm{P}}(g_{Z_t}) 
	- \mathcal{R}_{L, \mathrm{P}}^*
	\leq C  n^{-\frac{c_T\gamma(\alpha + 1)}{c_T\gamma(\alpha + 1)(2- (1 - \delta) \vartheta)+4d} },
	\end{align}
	where $c_T = 0.22$ and $C$ is a constant which will be specified in the proof later,
	independent of $n$ and only depending on constants $k$, $d$, $\vartheta$, $\delta$ and $\tau$.
\end{theorem}

Again,
if the noise exponent $\alpha$ is sufficiently large, 
we achieve asymptotically the optimal rate
for the best-scored random forest.
Moreover, if $\gamma$ or $\alpha$ is large, the rate is rather insensitive to the input dimension $d$ as well.

\subsection{Comments and Discussions}
This section presents some comments and discussions on the obtained theoretical results on the oracle inequalities, convergence rates for best-scored random trees and the learning rates for the ensemble forest.

From the theoretical perspective we notice that, 
on the one hand, rather than giving extra assumptions on the capacity of the function space, the size of function space in our algorithm is completely decided by the regularization term, while on the other hand, the approximation error is strictly calculated step by step according to the purely random scheme. 
Under certain assumptions in Section \ref{FundamentalAssumption}, when $\alpha$ go to infinity, 
we establish the asymptotically optimal learning rates \eqref{OOptimalLearningRate}
for the best-scored random trees, which is close to $\mathcal{O}(1 / n)$. 
Elementary analysis shows that asymptotically optimal rates for ensemble random forest, that is, $\mathcal{O}(1 / n)$, can be achieved.

As is already mentioned in the introduction, efforts have been paid to 
derive learning rates
for various kinds of random forests in the literature. 
Similar to our algorithm, \cite{Genuer12} and \cite{ArGe14} analyze purely random partitions independent of the data using bias-variance decomposition. 
More precisely, 
in the context of one-dimensional regression problems where
target functions are Lipschitz continuous,
\cite{Genuer12} shows that purely uniformly random trees/forests where the partitions were obtained by drawing $k$ random thresholds at random in $[0, 1]$
can both achieve minimax convergence rates $n^{-2/3}$.
Based on these models and their analysis,
\cite{ArGe14} obtains optimal convergence rates $n^{-4/5}$ over twice continuous differentiable functions in one-dimensional case for purely uniformly random forests and toy purely random partitions where the individual partitions corresponded to randomly shifts of the regular partition of $[0, 1]$ in $k$ intervals.
Note that in the latter work, boundaries and high-dimension cases are not considered. 
Concerning with high-dimension cases, under a sparsity assumption, \cite{Biau12} 
proves the convergence rate $n^{-3/(3 + 4 s \ln 2)}$
where $s$ denotes the number of active/strong variables. 
This rate is strictly faster than the commonly $d$-dimensional optimal rate $n^{-2/(2+d)}$ if $s \leq \lfloor 0.54d\rfloor$ and strictly slower than the $s$-dimensional optimal rate $n^{-2/(2 + s)}$. 
Furthermore, 
\cite{Gey14} proposes a penalized criterion and derive a risk bound inequality for the tree classifier generated by CART and thus
obtain convergence rates $\mathcal{O} ( n^{ - ( \alpha + 1 )/( \alpha + 2 ) } )$ which becomes asymptotically $\mathcal{O}(1 / n)$ if $\alpha \to \infty$.
Moreover, 
based on assumptions with respect to noise exponent $\alpha$ and smoothness of the target function $\gamma$,
\cite{Binev14} derives 
learning rates for certain recursive tree which
are of order $\mathcal{O} ( (\log n / n)^{ (1 + \alpha) \gamma/ ( (2+\alpha) \gamma + d )  } )$ when $\alpha \gamma (1-1/d) < 1$ and $0 < \gamma \leq 2$.
This rate can never better $\mathcal{O}((\log n / n)^{4/(6+d)})$. 
Last but not least, recently, 
if the target functions are Lipschitz continuous,
\cite{MoGa17} establishes the $d$-dimensional optimal rate $n^{-2/(2+d)}$
for online Mondrian Forests.

\section{Error Analysis}\label{sec::ErrorAnalysis}

In this section, we conduct error analysis by bounding the approximation error term and the sample error term respectively.

\subsection{Bounding the Approximation Error Term}

The following new result on bounding the approximation error term, 
which plays a key role in the learning theory analysis,
shows that
under certain assumptions on the amount of noise, 
the regularized approximation error
possesses a polynomial decay with respect to the 
regularization parameter $\lambda$.

\begin{proposition} \label{ApproxError}
	Let $L$ be the classification loss, $\beta$ be the margin-noise exponent as in Assumption \ref{MarginNoiseExponent} and $k$ be the number of candidate trees. Then for any fixed $\tau > 0$ and $\lambda > 0$, with probability $\mathrm{P}_Z$ at least $1-e^{-\tau}$, there holds for the $t$-the tree in the best-scored random forest that
	\begin{align*}
	\lambda p_{Z_t}^{*2} + \mathcal{R}_{L, \mathrm{P}}(g_{Z_t}^*) - \mathcal{R}_{L, \mathrm{P}}^*
	\leq c_{d, \beta} \ e^{\frac{8 d \tau \beta}{k (c_T \beta + 8 d)}} \lambda^{\frac{c_T \beta}{c_T \beta + 8 d}},
	\end{align*}
	where $c_T = 0.22$ and $c_{d, \beta}$ is a constant depending on $d$ and $\beta$.
\end{proposition}

\subsection{Bounding the Sample Error Term}

In learning theory, the sample error can be bounded by means of the 
Rademacher average and Talagrand inequality. 
Let $\mathcal{F}$ be a hypothesis space, the Rademacher average is defined as follows, 
see e.g., Definition 7.18 in \cite{StCh08}:

\begin{definition}[Rademacher Average]
	Let $\varepsilon_i$, $i = 1, \ldots, n$, be a Rademacher sequence with respect to some distribution $\nu$,
	that is, a sequence of i.i.d.~random variables such that 
	$\nu(\varepsilon_i = 1) = \nu(\varepsilon_i = -1) = 1/2$.
	The $n$-th empirical Rademacher average of $\mathcal{F}$ is
	defined as
	\begin{align*}
	\mathrm{Rad}(\mathcal{F}, n)
	:= \mathbb{E}_{\nu} \sup_{f \in \mathcal{F}} 
	\biggl| \frac{1}{n} \sum_{i=1}^n \varepsilon_i f(x_i) \biggr|.
	\end{align*}
\end{definition}

Recall the function space $\mathcal{T}_p$ defined as in \eqref{Tq}.  In the following analysis, for the sake of convenience,
we need to reformulate the definition of $\mathcal{T}_p$. 
Let $p \in \mathbb{N}$ be fixed.
Let $\mathcal{\pi}$ be a partition of $\mathcal{X}$ with number of splits $p$ 
and $\mathcal{\pi}_p$ denote the family of all partitions $\mathcal{\pi}$.
Furthermore, we define
\begin{align} \label{Bp}
\mathcal{B}_p
:= \biggl\{ B : B = \bigcup_{j \in J} A_j, J \subset \{0, 1, \ldots, p\}, A_j \in \mathcal{\pi}
\in  \mathcal{\pi}_p \biggr\}.
\end{align}
Then, for all $g \in \mathcal{T}_p$, 
there exists some $B \in \mathcal{B}_p$ such that
$g$ can be written as $g = \boldsymbol{1}_B - \eins_{B^c}$.
Therefore, $\mathcal{T}_p$ can be equivalently defined as
\begin{align} \label{Tp2}
\mathcal{T}_p := \bigl\{ \eins_B - \eins_{B^c} : B \in \mathcal{B}_p \bigr\}.
\end{align}

To establish the bounds on the sample error, 
we are encouraged to give a description of the capacity of the function space.

\begin{definition}[VC dimension]
	Let $\mathcal{B}$ be a class of subsets of $\mathcal{X}$ and $A \subset \mathcal{X}$ be a finite set.
	The trace of $\mathcal{B}$ on $A$ is defined by $\{ B \cap A : B \in  \mathcal{B} \}$. 
	Its cardinality is denoted by $\Delta^{\mathcal{B}}(A)$. 
	We say that $\mathcal{B}$ shatters $A$ if $\Delta^{\mathcal{B}}(A) = 2^{\#(A)}$,
	that is, if for every $\tilde{A} \subset A$, there exists a $B \subset \mathcal{B}$ such that
	$\tilde{A} = B \cap A$. For $k \in \mathbb{N}$, let 
	$$
	m^{\mathcal{B}}(k)  := \sup_{A \subset \mathcal{X}, \, \#(A) = k}  \Delta^{\mathcal{B}}(A).
	$$
	Then, the set $\mathcal{B}$ is a Vapnik-Cervonenkis (VC) class if there exists $k < \infty$ such that $m^{\mathcal{B}}(k) < 2^k$ and the minimal of such $k$ is called the \emph{VC dimension} of $\mathcal{B}$, 
	and abbreviated as $\mathrm{VC}(\mathcal{B})$.
\end{definition}

\begin{lemma} \label{VCdimension}
	The VC dimension of $\mathcal{B}_p$ in \eqref{Bp} can be upper bounded by $d p + 2$. 
\end{lemma}

\begin{definition}[Covering Number]
	Let $(\mathcal{X}, d)$ be a metric space, $A \subset \mathcal{X}$ and $\varepsilon > 0$.
	We call $\tilde{A} \subset A$ an $\varepsilon$-net of $A$ 
	if for all $x \in A$ there exists an $\tilde{x} \in \tilde{A}$ such that $d(x, \tilde{x}) \leq \varepsilon$.
	Moreover.
	the $\varepsilon$-covering number of $A$ is defined as
	\begin{align*}
	\mathcal{N} (A, d, \varepsilon) 
	= \inf \biggl\{ n \geq 1 : \exists x_1, \ldots, x_n \in \mathcal{X} \text{ such that } A \subset \bigcup_{i=1}^n B_d (x_i, \varepsilon) \biggr\},
	\end{align*}
	where $B_d(x, \varepsilon)$ denotes the closed ball in $\mathcal{X}$ centered at $x$ with radius $\varepsilon$.
\end{definition}
Let $\mathcal{B}$ be a class of subsets of $\mathcal{X}$, 
denote $\eins_{\mathcal{B}}$ as the collection of the indicator functions of all  $B \in \mathcal{B}$, that is, $\eins_{\mathcal{B}} := \{ \eins_B : B \in \mathcal{B} \}$.
Moreover, as usual, for any probability measure $Q$,
$L_2(Q)$ is denoted as the $L_2$ space with respect to $Q$ 
equipped with the norm $\|\cdot\|_{L_2(Q)}$.

\begin{lemma} \label{BpTpCoveringNumbers}
	Let $\mathcal{B}_p$ and $\mathcal{T}_p$ be defined as in \eqref{Bp} and \eqref{Tp2} respectively. Then, for all
	$0 < \varepsilon < 1$, there exists a universal constant $K$
	such that
	\begin{align} \label{BpCoveringNumber}
	\mathcal{N} \big( \eins_{\mathcal{B}_p}, \|\cdot\|_{L_2(Q)}, \varepsilon \big) 
	\le  K (d p + 2) (4 e)^{d p + 2} (1 / \varepsilon)^{2 (d p + 1)}
	\end{align}
	and
	\begin{align} \label{TpCoveringNumber}
	\mathcal{N} \bigl( \mathcal{T}_p, \|\cdot\|_{L_2(Q)}, \varepsilon \bigr) 
	\leq K (d p + 2) (4 e)^{d p + 2}(2 / \varepsilon)^{2 (d p + 1)}
	\end{align}
	hold for any probability measure $Q$.
\end{lemma}

\begin{definition}[Entropy Number]
	Let $(\mathcal{X}, d)$ be a metric space, $A \subset \mathcal{X}$ and $n \geq 1$ be an integer.
	The $n$-th entropy number of $(A, d)$ is defined as
	\begin{align*}
	e_n(A, d) = \inf \biggl\{ \varepsilon > 0 : \exists x_1, \ldots, x_{2^{n-1}} \in \mathcal{X} \text{ such that } A \subset \bigcup_{i=1}^{2^{n-1}} B_d(x_i, \varepsilon) \biggr\}.
	\end{align*}
\end{definition}
Denote
\begin{align} \label{rstar}
r^* : =  \inf_{g \in \mathcal{T}_{Z_t}} \ \lambda p^2(g) + \mathcal{R}_{L, \mathrm{P}} (g) - \mathcal{R}_{L, \mathrm{P}}^*.
\end{align}
For $r > r^*$, denote
\begin{align} 
&\mathcal{G}_r : = \{g \in \mathcal{T}_{Z_t} : \lambda p^2(g) + \mathcal{R}_{L, \mathrm{P}} (g) - \mathcal{R}_{L, \mathrm{P}}^* \le r\},
\nonumber
\\
&\mathcal{H}_r : = \{L \circ g - L \circ f_{L, \mathrm{P}}^* : g \in \mathcal{G}_r\},
\label{Hr}
\end{align}
where $L \circ g$ denotes the classification loss of $g$,
that is, $L \circ g(x, y) := L(x, y, g(x))$.

\begin{lemma} \label{HrEntropyNumber}
	Let $\mathcal{H}_r$ be defined as in \eqref{Hr}.
	Then, for all $\delta\in(0,1)$,
	the $i$-th entropy number of $\mathcal{H}_r$ satisfies
	\begin{align*}
	\mathbb{E}_{D \sim \mathrm{P}^n} \
	e_i \bigl( \mathcal{H}_r, \|\cdot\|_{L_2(D)} \bigr) 
	\leq c_{d,\delta} \ r^{1/(4\delta)} \lambda^{- 1/(4\delta)} i^{-1/(2\delta)},
	\end{align*}
	where $c_{d,\delta} := ((18 d) / (e \delta))^{1/2 \delta}$ is a constant depending on $d$ and $\delta$.
\end{lemma}

Having established the entropy number of the function set $\mathcal{H}_r$,
its Rademacher average can be bounded as follows:

\begin{proposition} \label{Rademacher}
	Let $\mathcal{H}_r$ be defined as in \eqref{Hr}. 
	Then for any $\delta \in (0,1)$, we have
	\begin{align*}
	\mathbb{E}_{D \sim \mathrm{P}^n} \ & \mathrm{Rad}_D (\mathcal{H}_r, n) 
    \leq \max \Bigl\{ c_{\delta,1} c_{d, \delta}^\delta V^{\frac{1-\delta}{2}}   
	\lambda^{-\frac{1}{4}} r^{\frac{2\vartheta-2\vartheta\delta+1}{4}} n^{-\frac{1}{2}},
	c_{\delta,2} c_{d, \delta}^{\frac{2\delta}{1+\delta}} \lambda^{-\frac{1}{2+2\delta}} 
	r^{\frac{1}{2+2\delta}} n^{-\frac{1}{1+\delta}} \Bigr\},
	\end{align*}
	where $c_{\delta,1}, c_{\delta,2} > 0$ are constants depending only on $\delta$, $c_{d,\delta}$ is the constant as in Lemma \ref{HrEntropyNumber}, $V$ is a constant depending on $\alpha$
	and $\vartheta = \alpha / (1 + \alpha)$ .
\end{proposition}

\section{A Counterexample} \label{sec::Counterexample}
Here, we present a counterexample to illustrate that in order to ensure the consistency of the forest, all dimensions must have the chance to be split. In other words, if the splits only occur on some predefined dimensions, this may leads to the inconsistency of the forest. 
We mention that the following counterexample is merely a simple case where we only consider a predefined dimension of size one.

First of all, we construct a special sample distribution which can be described by the following feature vector
\begin{align*}
x_i : = \sum_{j = 1}^d a_{i j} e_j, \quad a_{i j} \in \{ -1, 1 \}, \quad i = 1, \ldots, 2^d,
\end{align*}
where $e_j,\ j=1, \ldots, d$ are unit vectors, and $x_i,\ i=1,\ldots, 2^d$ can thus be viewed as feature points located at the vertexes of a $d$-dimension cube on which all probability mass is assumed to be concentrate. Points are then labeled as follows:
\begin{align*}
y_i : = 
\begin{cases}
1 & \mathrm{if} \ \sum_{j = 1}^d  \eins_{\{a_{i j} = 1\}} \ \mathrm{is\ even}, \\
-1 & \mathrm{otherwise}.
\end{cases}
\end{align*}
It can be seen that we have the Bayes risk $\mathcal{R}^* = 0$ in this case.
Samples are chosen in the form of pair $(x_i, y_i)$ uniformly at random with replacement from the above mentioned sample distribution and therefore form a data set $D$ of size $n$.

Secondly, we scrutinize results of performing splits only from one dimension to form a decision tree. 
If we project all data to that dimension, say the $j$th dimension,
from which we will perform the splits,
then the classification rule will be based on the projection data.
Therefore, in this case, no matter where the chosen cut point is, the classification result of any point 
$x = \sum_{j = 1}^d a_j e_j$ can be obtained by
\begin{align*}
\hat{f}_{D_j} (x) = 
\begin{cases}
1 & \sum_{k = 1}^n y_k \eins_{\{a_{k j} = a_j\}} > 0, \\
-1 & \mathrm{otherwise},
\end{cases}
\end{align*}
which is named as the classifier of the $j$th dimension.
Without loss of generality, we assume that 
samples located on different sides of the splits are classified into different classes, i.e.~
$\hat{f}_{D_j} (x) \hat{f}_{D_j} (x') = a_j a'_j$. 
For other cases, the analysis can be conducted in the same way.

The above process provides us a way to construct one decision tree and therefore we are able to develop a forest with each tree built in that way. The constructed $m$-tree forest consists of classifiers of different dimensions where the number of classifiers from the same dimension, say the $j$ dimension, is denoted as $m_j$, so that we have $m = \sum_{j=1}^d m_j$.
As a result, the forest classifier is still the majority vote of trees and presented as
\begin{align*}
\hat{f}_D (x) = \sign \Big(\sum_{j = 1}^m m_j \hat{f}_{D_j} \Big).
\end{align*}
However, we illustrate that the above forest classifier cannot be consistent.

On one hand, if $d$ is even, 
for any feature point $x = \sum_{j = 1}^d a_j e_j$, 
we assume that there exists an even number $k$ such that
$x$ can be presented by
\begin{align*}
x = \sum_{j = 1}^k e_{(j)} - \sum_{j = k + 1}^d e_{(j)},
\end{align*} 
where $\{ e_{(1)}, \ldots, e_{(d)} \}$ is a rearrangement of $\{ e_1, \ldots, e_d \}$ so that the first $k$ elements in $\{ e_{(1)}, \ldots, e_{(d)} \}$ corresponding to the elements in $\{ e_1, \ldots, e_d \}$ who have $a_j = 1$.
As a result, the true label of $x$ is $y=1$.
Since
$d - k$ is also an even number, 
we can find a corresponding point
\begin{align*}
x' = \sum_{j = k + 1}^{d} e_{(j)} - \sum_{j = 1}^k e_{(j)}.
\end{align*} 
whose true label is also $Y' = 1$. 
Regardless of the classification from any one dimension, $x$ and $x'$ will always be assigned to different categories. In other word, whenever $x$ is labeled $1$ based on projection data $D_j$, $x'$ will be labeled as $-1$, so the final prediction result will always be different, which leads to a misclassification ratio of  $50\%$.  
The analysis is the same for cases where the true label of $x$ and $x'$ is $-1$.

On the other hand, 
if $d$ is odd, for any feature point $x = \sum_{j = 1}^d a_j e_j$, 
we assume that there exists a dimension, say the $j$th dimension, whose coefficient is $a_j = 1$ and an odd number $k$ such 
\begin{align*}
x = e_j + \sum_{i = 1}^k e_{(i)} - \sum_{i = k + 1}^{d - 1} e_{(i)}
\end{align*}
where $\{ e_{(1)}, \ldots, e_{(d-1)} \}$ is a rearrangement of $\{ e_1, \ldots, e_{j-1}, e_{j+1}, \ldots, e_d \}$ similar as above. Therefore the true label of $x$ is $y = 1$. Then we can also find a corresponding point
\begin{align*}
x' = e_j + \sum_{i = k + 1}^{d - 1} e_{(i)} - \sum_{i = 1}^k e_{(i)},
\end{align*}
whose true label is also $y' = 1$ for
$d - 1 -k$ is also an odd number. 
And for an even number $k'$, we can find two more points which are 
\begin{align*}
x'' = e_j + \sum_{i = 1}^{k'} e_{(i)} - \sum_{i = k' + 1}^{d - 1} e_{(i)}
\end{align*}
whose true label is $Y'' = -1$ and
\begin{align*}
x''' = e_j + \sum_{i = k' + 1}^{d - 1} e_{(i)} - \sum_{i = 1}^{k'} e_{(i)},
\end{align*}
whose true label is also $Y'''= -1$.
When $x$ and $x'$ are correctly classified, $x''$ and $x'''$ must be misclassified since if not, a contradiction will happen on the decision made on the $j$th dimension, which also means that the misclassification ratio of the forest method is $50\%$.

As a result, the above forest can not be consistent in this case. Moreover, this example can also be extended to the cases where data is vertically projected to the space of dimension $d_0 \le d - 1$.

\section{Experimental Performance}\label{sec::experiments}
In this section, we discuss the model selection problem of the best-scored random forest by performing numerical experiments and give comparisons between our algorithm and other classification methods.

\subsection{Construction of Forests}
When it comes to model selection of random forest,
the partition of the data is what we should pay special attention to.
To this end, for our best-scored random forest, there is a need to first pick up appropriate values for the number of candidates which is $k$ when choosing each decision tree in the forest, and $m$ for the total number of trees in the forest according to data sets at hand (\cite{Probst18} gives some advice on the choice of number of trees in random forest). 
Moreover, 
we generate an appropriate number of splits $p$
for the construction of decision trees according to the data set.
We mention here that in our experiments, all data sets are divided into two parts which are training data and testing data. These two account for $70\%$  and $30\%$ of the overall data volume, respectively. 

There is one vital thing to emphasize here.
Concerning that purely random tree may face the dilemma where the effective number of splits is relatively small, we propose a new effective partition method called the adaptive random partition which improves the original purely random partition procedures. For the original partition, $L_i$ in the random vector $Q_i := (L_i, R_i, S_i)$ (please refer to Section \ref{PRTF}) denotes the randomly chosen leaf to be split at the $i$-th step of tree construction. However, since this choice of $L_i$ does not make any use of the sample information, it may suffer from over-splitting on sample-sparse nodes and under-splitting on sample-dense nodes. A wise improvement need to use sample information to some extent without losing the randomness of the choice of nodes. Therefore, we propose that when choosing a to be split node, we first randomly select one sample point from the training data set, and then choose the node which that sample point belongs to as $L_i$. According to the fact that when randomly picking sample point from the whole training data set, 
nodes with more samples will be more likely to be selected, we obtain a partition where there are more splits applied to sample-dense area, while fewer splits applied to sample-sparse area. In this way, we develop an adaptive random partition.

\begin{figure*}[htbp]
	\centering
	\begin{minipage}[b]{0.32\textwidth}
		\centering
		\includegraphics[width=\textwidth]{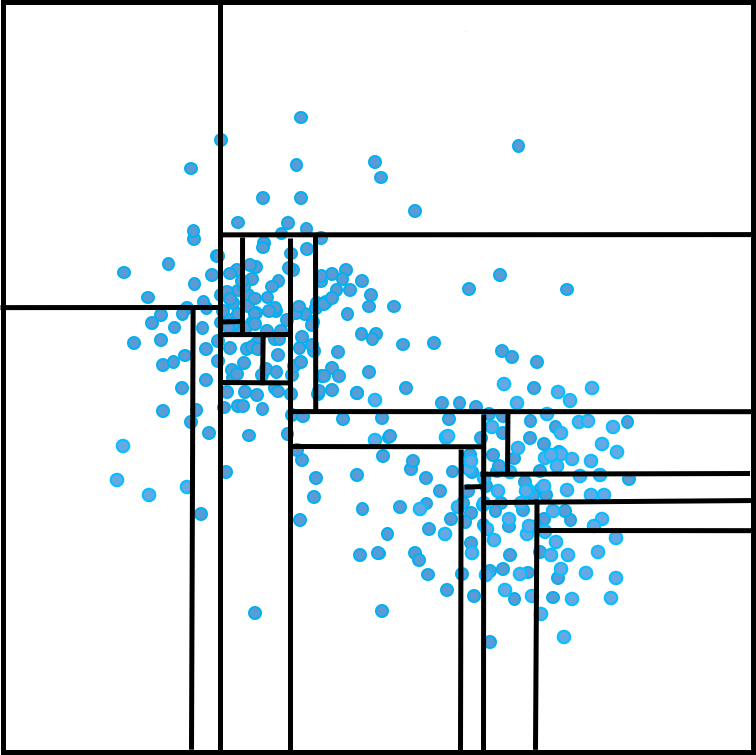}
		Adaptive Method
		\centering
	\end{minipage}
	\qquad
	\begin{minipage}[b]{0.32\textwidth}
		\centering
		\includegraphics[width=\textwidth]{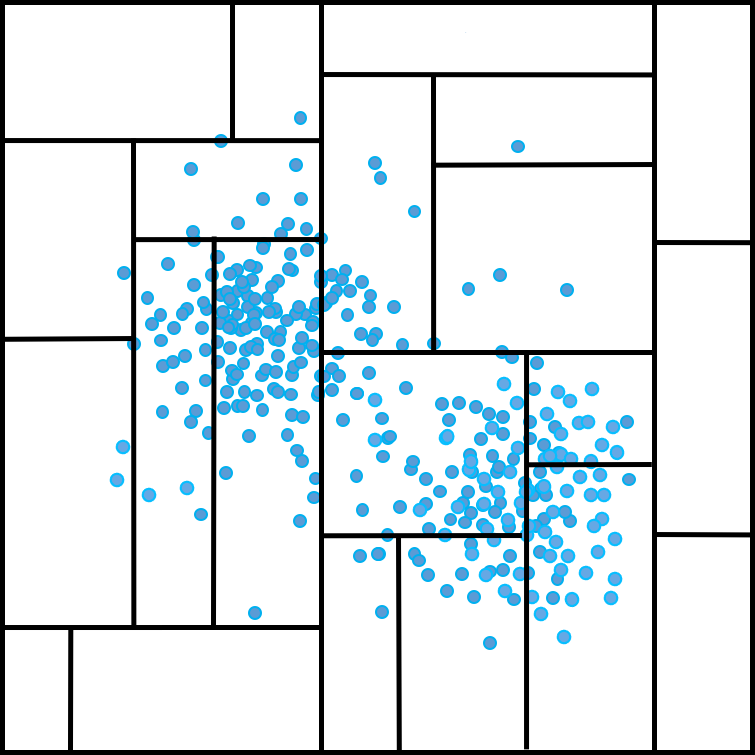}
		Purely Random Method
	\end{minipage}
\end{figure*}

Now we present the details of the construction of one random tree.
To begin with, we generate $k$ $p$-splitting adaptive random partitions. The main objective of the following work is to choose the partition with the best classification performance from $k$ candidates via a $10$-fold cross-validation.
Now we discuss the first round of the $10$-fold cross-validation.
Based on the training set of cross-validation, for each of the $k$ partitions, the corresponding classifier is derived by labeling each cells of that partition with $1$ or $-1$ according to the majority votes of the samples which fall into that cell.
Subsequently, we derive the validation errors which are the empirical risks based on the validation set of cross-validation for all $k$ base classifiers.
Once traversing all ten rounds, for each classifier, we are then able to calculate its average validation error for ten rounds.
As a result, we can choose the partition from all $k$ candidates with the smallest average validation error to be the exact partition for one tree.
Furthermore, by giving labels to all the cells of the chosen partition in accordance with the majority votes basing on the training data, which accounts for the $70\%$ of the overall data volume as mentioned before, we finally manage to construct one tree in the forest.
That is what we call the \emph{best-scored tree}. 
According to this construction approach, we obtain all $m$ trees which consequently form the \emph{best-scored random forest} for a fixed $(k, m)$ pair.

In summary, we successfully establish an appropriate \emph{best-scored random forest} for data at hand.

\subsection{Experimental Setup}
Having developed the \emph{best-scored random forest}, it is high time for us to conduct comparisons between our algorithm and different mainstream classification algorithms.
Test errors of a forest or classifiers from other approaches are evaluated by the empirical risk based on the testing data. 
In our experiments, we conduct comparisons among
baseline algorithms which are 
Breiman's random forest (RF),
Extremely Randomized Trees (ExtRa),
$k$-Nearest Neighbors ($k$-NN), and 
Support vector machines (SVMs), 
for those algorithms are well established.  
\begin{itemize}
	\item 
RF: The idea can be traced back to \cite{breiman2001random}. The leading parameters are mainly the number of trees in the forest, the number of samples for the construction of one tree, the number of possible features for splitting at each node of each tree and number of splits.
   \item
ExtRa: This algorithm proposed by \cite{geurts2006extremely}, is also a tree-based ensemble method for supervised learning while it selects splits, both attribute and cut-point, totally or partially at random. There are two parameters in the algorithm: the number of attributes randomly selected at each node and the minimum sample size for splitting a node $n_{min}$. 
	\item 
$k$-NN: Here we use \cite{Aha91}'s version of $k$-nearest neighbor and the leading parameter in the algorithm is the number of nearby points $k$. The idea of $k$-NN is inspired by \cite{Cover67};
	\item 
SVMs: The leading parameters in support vector machines are the regularization parameter and bandwidth. The idea originates from \cite{Cortes95}; 
\end{itemize}
For these algorithms, we simply used the Python-package scikit-learn with default settings. 

It is worth pointing out that our best-scored random forest have much more free parameters tunable compared to other methods. To be specific, these free parameters include the number of candidates  for choosing each decision tree $k$,  the total number of trees in the forest $m$, the number of splits for trees in the forest $p$, 
the parameter $a \in  [0,0.5]$ in the uniform distribution $\mathrm{Unif}[0.5-a, 0.5+a]$ for selecting the cut point.

\subsection{Practical Performance Analysis}

In our experiments, we conduct empirical comparisons on the classification accuracy among 
Breiman's random forest (RF),
Extremely Randomized Trees (ExtRa),
$k$-Nearest Neighbors ($k$-NN), and 
Support vector machines (SVMs), 
and our best-scored random forest (BRF) on the following UCI data sets:
\begin{itemize}
	\item \textbf{monks: } 
	The MONK's Problems is a data set which consists of three subsets. It is derived from a domain in which each training example is represented by six discrete-valued attributes. The second subset is used here to learn a binary function defined over this domain, from a sample of training examples of this function. We randomly split each data set into a training set containing 70\% of the total observations and a test set containing the remaining 30\% of the observations.
	\item \textbf{bcw:} 
	Breast Cancer Wisconsin contains data collected by Dr. Wolberg in his clinical cases throughout 1989 and 1991. Since samples arrive periodically, the database reflects this chronological grouping of the data. The whole data set consists of 699 samples of dimension 11 where the ratio of number of samples in the training set and the testing set is $7 : 3$. 
	\item \textbf{energy:} 
	Energy efficiency comprises 768 samples from energy analysis using 12 different building shapes simulated in Ecotect. The 8-dimensional input data represent attributes such as Wall Area, Roof Area, Overall Height and Orientation of the building, etc. As for outputs, we use the second one out of the two response variables for numerical study. The main learning task is to predict Cooling Load of the building when the input observation is available.
	\item \textbf{ILPD:}
	The data set Indian Liver Patient Data set contains 416 liver patient records and 167 non liver patient records. The data set was collected from 441 male patients and 142 female patients in north east of Andhra Pradesh, India. The data set consists of ten attributes such as Age, Gender, etc and the response variable {\tt Selector} is a class label used to divide into groups liver patient or not. 
	\item \textbf{ozone:}
	The database ozone available on UCI contains 2536 observations of dimension 73. Two ground ozone level data including the eight hour peak data and the one hour peak data is collected from 1998 to 2004 at the Houston, Galveston and Brazoria area. The learning goal is to predict the total column of ozone at any unobserved location.
	\item \textbf{Statlog:}
	This data set containing 690 observations of dimension 14 concerns credit card applications. It is interesting because there is a good mix of attributes -- continuous, nominal with small numbers of values, nominal with larger numbers of values and also a few missing values. The data is used to predict the response A1 and to notify, all attribute names and values have been changed to meaningless symbols to protect confidentiality of the data.
	
\end{itemize}
The following table summarizes the classification errors on test data of the UCI data sets. 
Here, $n$ and $d$ denote sample size and dimension, respectively.
We mention that
the hyper-parameters of each algorithm are selected by $3$-fold cross-validation, and
all the experiments are repeated for $50$ times.
\begin{table}[h] 
	\setlength{\tabcolsep}{7.5pt}
	\centering
	\captionsetup{justification=centering}
	\caption{Classification Error on Test Data of UCI Data Sets (\%)}
	\label{ClassificationErrorTable}
	\begin{tabular}{@{}l ||  c | ccccc@{}}
		\toprule
		\text{Data Sets} & $(n,d)$  & \text{RF} & \text{ExtRa}  & \text{$k$-NN} & \text{SVMs} & \text{BRF} \\ \midrule
        \hline
		\multirow{2}*{{\tt monks}}
		& \multirow{2}*{$(601, 6)$} 
		& $0.6613^{*}$ & $0.6597^{*}$ & $0.6589^{*}$ & $0.6587^{*}	$ & $\textbf{0.6681}$  \\ 
		& & $(\pm 0.0017)$ & $(\pm 0.0065)$ & $(\pm 0.0044)$ & $(\pm 0.0052)$ & $(\pm 0.0024)$  \\ 
		\hline
		\multirow{2}*{{\tt bcw}}
		& \multirow{2}*{$(699, 9)$} 
		& $0.9683^{*}$ & $0.9698^{*}$ & $0.9632^{*}$ & $0.9657^{*}$ & $\textbf{0.9720}$  \\ 
		& & $(\pm 0.0099)$ & $(\pm 0.0104)$ & $(\pm 0.0104)$ & $(\pm 0.01217)$ & $(\pm 0.0104)$  \\ 
		\hline
		\multirow{2}*{{\tt energy}}
		& \multirow{2}*{$(768, 8)$} 
		& $0.9156^{*}$ & $0.9174$ & $0.8956^{*}$ & $0.9169$ & $\textbf{0.9204}$  \\ 
		& & $(\pm 0.0168)$ & $(\pm 0.0180)$ & $(\pm 0.0163)$ & $(\pm 0.0181)$ & $(\pm 0.0153)$  \\ 
		\hline
		\multirow{2}*{{\tt ILPD}}
		& \multirow{2}*{$(583, 9)$} 
		& $0.6969^{*}$ & $0.7041^{*}$ & $0.6954^{*}$ & $0.7006^{*}$ & $\textbf{0.7113}$  \\ 
		& & $(\pm 0.0299)$ & $(\pm 0.0285)$ & $(\pm 0.0273)$ & $(\pm 0.0232)$ & $(\pm 0.0128)$  \\ 
		\hline
		\multirow{2}*{{\tt ozone}}
		& \multirow{2}*{$(2536, 72)$} 
		& $0.9702^{*}$ & $0.9707^{*}$ & $0.9698^{*}$ & $0.9698^{*}$ & $\textbf{0.9711}$  \\ 
		& & $(\pm 0.0012)$ & $(\pm 0.0010)$ & $(\pm 0.0023)$ & $(\pm 0.0029)$ & $(\pm 0.0000)$  \\ 
		\hline
		\multirow{2}*{{\tt Statlog}}
		& \multirow{2}*{$(690, 14)$} 
		& $0.6717$ & $0.6670^{*}$ & $0.6334^{*}$ & $0.6709$ & $\textbf{0.6727}$  \\ 
		& & $(\pm 0.0202)$ & $(\pm 0.0205)$ & $(\pm 0.0302)$ & $(\pm 0.0145)$ & $(\pm  0.0124)$  \\ 
		\bottomrule
	\end{tabular}
\begin{tablenotes}
	\footnotesize
	\item{*} indicates that our classification accuracy improvement over other algorithms is also supported by Wilcoxon signed ranked test (\cite{demsar2006statistical}) which assures the statistical significance at level $0.05$.
\end{tablenotes}
	\label{ErrorTable}
\end{table}

Having scrutinized the classification errors presented in Table \ref{ClassificationErrorTable}, we come to the conclusion that our best-scored random forest algorithm performs best among the above mentioned state-of-art classification algorithms.
Finally, we mention that experimental results presented in Table \ref{ClassificationErrorTable} are those we have temporarily tuned. Interested readers can further tune the free parameters and more accurate results could be obtained.

\section{Proofs}\label{sec::proofs}

To prove Proposition \ref{ApproxError}, we need the following result which follows from Lemma 6.2 in \cite{Devroye86a}.

\begin{lemma} \label{Saturation}
	For a binary search tree with $n$ nodes, denote the saturation level $S_n$ as the number of full levels of nodes in the tree. Then for $k \geq 1$ and $\log n > k + \log (k + 1)$, there holds
	\begin{align*}
	\mathrm{P} (S_n < k + 1) 
	\leq \Big( \frac{k + 1}{n} \Big) \Big( \frac{2e}{k} \log \Big( \frac{n}{k + 1} \Big) \Big)^k.
	\end{align*}
\end{lemma}

\begin{proof}[Proof of Proposition \ref{ApproxError}]
	For the $t$-th tree in the forest, for any fixed $\ell \in \{ 1, \ldots, k \}$, we need to bound the approximation error of $g_{Z_{\ell t}}^*$ defined in \eqref{GLP} first. 
	To begin with, we denote $g_{Z_{\ell t}, p}^*$ as the function that minimize $\mathcal{R}_{L, \mathrm{P}}(g) - \mathcal{R}_{L, \mathrm{P}}^*$ in $\mathcal{T}_{Z_{\ell t}, p}$ with number of splits $p$, where $\mathcal{T}_{Z_{\ell t}, p}$ is the function set as in \eqref{SpaceZl}. We then denote the collection of cells associated with the formation of  $g_{Z_{\ell t}, p}^*$ as $\mathcal{A}_Z : = \{ A_j , j = 0, \ldots, p \}$ and separate $\mathcal{A}_Z$ into two parts by whether the cells in $\mathcal{A}_Z$ intersect the decision line $\mathcal{X}_0$, that is
	\begin{align*}
	\mathcal{A}_1 & : = \{ A \in \mathcal{A}_Z \ | \ A \cap \mathcal{X}_1  = \emptyset \ \mathrm{or} \ A \cap \mathcal{X}_{-1} = \emptyset\}, \\
	\mathcal{A}_2 & : = \{ A \in \mathcal{A}_Z \ | \ A \cap \mathcal{X}_1  \neq \emptyset \ \mathrm{and} \ A \cap \mathcal{X}_{-1} \neq \emptyset\},
	\end{align*}
	where $\mathcal{X}_0$, $\mathcal{X}_1$ and $\mathcal{X}_{-1}$ are defined as in \eqref{DistanceSet}. 
	Then, Example 3.8 in \cite{StCh08} tells us that
	\begin{align*}
	\mathcal{R}_{L, \mathrm{P}} (g_{Z_{\ell t}^*, p}) - \mathcal{R}_{L, \mathrm{P}}^*
	& = \int_{\mathcal{X}_1 \Delta \{ g_{Z_{\ell t}^*, p}(x) = 1\}} | 2 \eta(x) -1 | \ d \mathrm{P}_X (x) \\
	& = \sum_{A \in \mathcal{A}_1} \int_{(\mathcal{X}_1 \Delta \{ g_{Z_{\ell t}, p}^*(x) = 1\}) \cap A} | 2 \eta(x) -1 | \ d \mathrm{P}_X (x) \\
	& \phantom{=} + \sum_{A \in \mathcal{A}_2} \int_{(\mathcal{X}_1 \Delta \{ g_{Z_{\ell t}, p}^*(x) = 1\}) \cap A} | 2 \eta(x) -1 | \ d \mathrm{P}_X (x).
	\end{align*}
	For all $A \in \mathcal{A}_1$, we have $A \cap \mathcal{X}_1= \emptyset$ or $A \cap \mathcal{X}_{-1} = \emptyset$ and thus $A \subset \mathcal{X}_1$ or $A \subset \mathcal{X}_{-1}$. By taking $g_{Z_{\ell t}, p}^* (x) = 1$ for $x \in A \subset \mathcal{X}_1$ and $g_{Z_{\ell t}, p}^* (x) = -1$ for $x \in A \subset \mathcal{X}_{-1}$, we have
	\begin{align*}
	\sum_{A \in \mathcal{A}_1} \int_{(\mathcal{X}_1 \Delta \{ g_{Z_{\ell t}, p}^*(x) = 1\}) \cap A} | 2 \eta(x) -1 | \ d \mathrm{P}_X (x) = 0.
	\end{align*}
	Therefore, only those cells which intersect the decision line contribute to the approximation error term. Further, we decompose the error by the diameter of the cells in $\mathcal{A}_2$. Denote
	\begin{align*}
	\mathcal{A}_3 := \{ A \in \mathcal{A}_2 \ | \ \mathrm{diam}(A) \le h \} 
	\qquad \mathrm{and} \qquad 
	\mathcal{A}_4 & := \{ A \in \mathcal{A}_2 \ | \ \mathrm{diam}(A) > h \},
	\end{align*}
	where $\mathrm{diam}(A)$ is the diameter of $A$.
	In the following proof, we consider the $L_1$-norm which leads to the definition of the diameter of the cell as
	$
	\mathrm{diam}(A) = \sum_{i=1}^d V_i (A), 
	$
	where $V_i (A)$ denotes the length of the $i$-th dimension of the rectangle cell $A$.
	then
	\begin{align} \label{ErrorDec}
	\mathcal{R}_{L, \mathrm{P}} (g_{Z_{\ell t}^*, p}) - \mathcal{R}_{L, \mathrm{P}}^*
	&= \sum_{A \in \mathcal{A}_2} \int_{(\mathcal{X}_1 \Delta \{ g_{Z_{\ell t}, p}^*(x) = 1\}) \cap A} | 2 \eta(x) -1 | \ d \mathrm{P}_X (x) 
	\nonumber \\
	& \le \sum_{A \in \mathcal{A}_2} \int_{A} | 2 \eta(x) -1 | \ d \mathrm{P}_X (x) 
	\nonumber \\
	& = \sum_{A \in \mathcal{A}_3} \int_{A} | 2 \eta(x) -1 | \ d \mathrm{P}_X (x) 
          + \sum_{A \in \mathcal{A}_4} \int_{A} | 2 \eta(x) -1 | \ d \mathrm{P}_X (x).
	\end{align}
	
	Let us now consider the first term in the decomposition \eqref{ErrorDec}. For $A \in \mathcal{A}_3$, the cell intersects the decision line and the diameter of the cell is less than $h$. Then for any $x \in A \in \mathcal{A}_3$, the distance to the decision boundary $\Delta(x)$ defined in \eqref{DistanceBound} satisfies $\Delta(x) \leq h$.
	Using Assumption \ref{MarginNoiseExponent}, we get
	\begin{align} \label{DecompositionTermOne}
	\sum_{A \in \mathcal{A}_3} \int_{A} |2 \eta(x) - 1| \ d\mathrm{P}_X(x)
	\leq \int_{\Delta(x) \leq h} |2 \eta(x) - 1| \, d\mathrm{P}_X(x)
	\leq c_\beta h^\beta,
	\end{align}
	where $c_\beta$ is the constant as in Assumption \ref{MarginNoiseExponent}.
	For the second term in the decomposition \eqref{ErrorDec}, 
	elementary considerations imply that
	\begin{align} \label{Approxempty} 
	\mathrm{P}_Z ( \{ A \ | \ A \in \mathcal{A}_4 \} = \emptyset)
	& =  \mathrm{P}_Z ( \{ A \in \mathcal{A}_2 \ | \  \text{diam}(A) > h \} = \emptyset) 
	\nonumber \\
	& \geq \mathrm{P}_Z ( \{ A \in \mathcal{A}_Z \ | \  \text{diam}(A) > h \} = \emptyset) 
	\nonumber \\
	& = \mathrm{P}_Z ( \forall A \in \mathcal{A}_Z : \text{diam}(A) \leq h ) 
	\nonumber \\
	& = \mathrm{P}_Z \Bigl( \max_{A \in \mathcal{A}_Z} \mathrm{diam}(A) \leq h \Bigr).
	\end{align}
	Then by Markov's inequality, we obtain
	\begin{align} \label{DecompositionMarkov}
	\mathrm{P}_Z \Big( \max_{A \in \mathcal{A}_Z} \mathrm{diam}(A) \leq h \Big)
	& \geq 1 - h^{-1} \mathbb{E}_Z \Big( \max_{A \in \mathcal{A}_Z} \ \mathrm{diam}(A) \Big) 
	\nonumber \\
	& = 1 - h^{-1} \mathbb{E}_Z \Big( \max_{A \in \mathcal{A}_Z} \sum_{i=1}^d V_i (A) \Big)  
	\nonumber \\
	& \geq 1 - h^{-1} \sum_{i=1}^d \mathbb{E}_Z \Big( \max_{A \in \mathcal{A}_Z} V_i (A) \Big).
	\end{align}
	Recall that $Z$ is defined by $(Q_0, Q_1, \ldots, Q_p, \ldots)$ where $Q_i = (L_i, R_i, S_i)$, $i=0, 1, \ldots$ in Section \ref{PRTF}. This shows that the randomness of $Z$ actually results from three aspects: randomly selecting nodes, randomly picking dimensions, randomly determining cut points. 
	The above analysis of the variable $Z$ describes the exact constructing process of the tree entirety.
	In order to ensure the feasibility of the calculation of expectation respect to $Z$ in \eqref{DecompositionMarkov}, we need to conduct analysis supposing that the tree is well-established.
	In particular, for each dimension, we only consider one cell that has the longest side length in its respective dimension. To mention, due to the symmetry of dimensions, it suffices to first focus on one dimension, e.g.~the $i$-th dimension and we denote the length of the $i$-th dimension of the corresponding cell as $\max_{A \in \mathcal{A}_Z} V_i (A) = : V_Z$. 
	To calculate $\mathbb{E}_Z(V_Z)$, we do not have to know the exact constructing procedures of the tree entirety. Instead, we still consider from three aspects which is intrinsically corresponding to above, but from a different view: the total number of splits that generates that specific rectangle cell during the construction, $T_Z$; 
	the number of splits which come from the $i$-th dimension in $T_Z$, $K_Z$ and $K_Z$ follows the binomial distribution $\mathcal{B}(T_Z, 1/d)$; 
	and proportional factors $U_1, U_2, \ldots, U_{K_Z}$ which are independent and identically distributed from $\mathrm{Unif}[0,1]$.
	According to the above statements, we come to the conclusion that the expectation with regard to $Z$ can be decomposed as $\mathbb{E}_Z = \mathbb{E}_{T_Z}\mathbb{E}_{K_Z|T_Z}\mathbb{E}_{U_1\ldots U_{K_Z}|K_Z}$.	
	Therefore, the expectation in the last step in \eqref{DecompositionMarkov} can be further analyzed as follows:
	\begin{align*}
	\mathbb{E}_Z V_Z
	& \leq \mathbb{E}_{T_Z} \biggl( 
	\mathbb{E}_{K_Z} \bigg( 
	\mathbb{E}_{U_1\ldots U_{K_Z}} \biggl( \prod_{j=1}^{K_Z} \max \{ U_j, 1 - U_j \} \Big| K_Z \biggr) \Big| T_Z \biggr) \biggr)
	\\
	& = \mathbb{E}_{T_Z} \Big( \mathbb{E}_{K_Z} \Big( \Big( \mathbb{E}_U \big( \max \{ U, 1 - U \} \big) \Big)^{K_Z}  \big| T_Z \Big)\Big)
	= \mathbb{E}_{T_Z} \Big( \mathbb{E}_{K_Z} \Big( \big( 3 / 4 \big)^{K_Z} \big| T_Z \Big) \Big) 
	\\
	& = \mathbb{E}_{T_Z} \biggl( 
	\sum_{K_Z=1}^{T_Z} {T_Z \choose K_Z} 
	\biggl( \frac{3}{4} \biggr)^{K_Z} 
	\biggl( \frac{1}{d} \biggr)^{K_Z} 
	\biggl( 1 - \frac{1}{d} \biggr)^{T_Z - K_Z} \biggr)
	\\
	& = \mathbb{E}_{T_Z} \biggl( 1 - \frac{1}{d} + \frac{3}{4d} \biggr)^{T_Z}
	= \mathbb{E}_{T_Z} \biggl( 1 - \frac{1}{4d} \biggr)^{T_Z}.
	\end{align*}
	Observing that when the underlying partition rule $Z$ has number of splits $p$, 
	the partition tree is statistically related to a random binary search tree with $p+1$ external nodes and $p$ internal nodes. 
	Then, Lemma \ref{Saturation} states that for $k \ge 1$ and $\log p > k + \log (k + 1)$,
	\begin{align*}
	\mathrm{P} (S_p < k + 1) 
	\leq \biggl( \frac{k + 1}{p} \biggr) 
	\biggl( \frac{2 e}{k} \log \biggl( \frac{p}{k + 1} \biggr) \biggr)^k,
	\end{align*}
	where $S_{p}$ is the \emph{saturation level}. In our setting, $S_{p}$ can be viewed as the minimal number of splits that generates any $A \in \mathcal{A}$. Now taking $k = \lfloor c_T \log p \rfloor$ where $c_T < 1$ and $c_T (1 + \log (2 e / c_T)) < 1$, simple calculation shows that
	\begin{align*}
	\mathrm{P} (T_Z < \lfloor c_T \log p \rfloor + 1)
	\leq \mathrm{P} (S_p < \lfloor c_T \log p \rfloor + 1)
	\leq K p^{c_T (1 + \log (2e / c_T)) - 1},
	\end{align*}
	where $K$ is a universal constant. Consequently for any $A \in \mathcal{A}$, we have
	\begin{align*}
	\mathbb{E}_Z V_Z
	& \leq \mathbb{E}_{T_Z} \Bigl( \Bigl( 1 - \frac{1}{4d} \Bigr)^{T_Z} \eins_{\{ T_Z < \lfloor c_T \log p \rfloor + 1\}} \Bigr)
	+ \mathbb{E}_{T_Z} \Bigl( \Bigl( 1 - \frac{1}{4d} \Bigr)^{T_Z} \eins_{\{T_Z \geq \lfloor c_T \log p \rfloor + 1\}} \Bigr)
	\\
	& \leq K p^{c_T (1 + \log (2e / c_T)) - 1} + \big( 1 - 1/(4d) \big)^{c_T \log p}
	\\
	& \leq K p^{c_T (1 + \log (2e / c_T)) - 1} + p^{- c_T / (4d)},
	\end{align*}
	where the last inequality follows from the fact that $1 - 1 / x < e^{-x}$ for all $x > 1$. Since the function $f(c_T) = 1 - c_T (1 + \log (2e / c_T)) - c_T / (4d)$ is monotone decreasing on $(0, 1)$ for all $d$, numerical computation shows that the largest constant for which $1 - c_T (1 + \log (2e / c_T)) > c_T / (4d)$ holds for all $d \ge 1$ cannot be greater than $0.22563$.
	Therefore, taking $c_T = 0.22$ and $C = K + 1$, there holds
	$\mathbb{E}_Z V_Z \le C p^{- c_T / (4d)}$.
	Therefore, we obtain that
	\begin{align} \label{DecompositionTwo}
	\sum_{i=1}^d \mathbb{E}_Z \Big( \max_{A \in \mathcal{A}_Z} V_i (A) \Big)
	\leq C d p^{- c_T / (4d)}.
	\end{align}
	Combining \eqref{Approxempty}, \eqref{DecompositionMarkov} and \eqref{DecompositionTwo}, we have 
	\begin{align} \label{EstimationTwo}
	\mathrm{P}_Z ( \{ A \ | \ A \in \mathcal{A}_4 \} = \emptyset ) 
	\geq 1 - C d h^{-1} p^{- c_T / (4d)}.
	\end{align}
	In other words, with probability at least $1 - C d h^{-1} p^{- c_T / (4d)}$, the second term in the error decomposition \eqref{ErrorDec} vanishes. 
	
	Now, the estimation \eqref{EstimationTwo} together with \eqref{DecompositionTermOne} yields that 
	\begin{align*}
	\mathcal{R}_{L, \mathrm{P}}(g_{Z_{\ell t}, p}^*) - \mathcal{R}_{L, \mathrm{P}}^*
	\leq c_{\beta}h^{\beta}
	\end{align*}
	holds with probability at least $1 - C d h^{-1} p^{- c_T / (4d)}$.
	By taking $e^{-\theta} := C d h^{-1} p^{- c_T / (4d)}$ and adding the regularization term to both sides of the above inequality, 
	simple calculation shows that with probability $\mathrm{P}_Z$ at least $1-e^{-\theta}$, there holds
	\begin{align} \label{ResultOne}
	\lambda p^2
	+ \mathcal{R}_{L, \mathrm{P}}(g_{Z_{\ell t}^*, p}) - \mathcal{R}_{L, \mathrm{P}}^*
	\leq \lambda p^2 + c_{\beta} \Big(C d e^{\theta} p^{- c_T / (4d)} \Big)^{\beta} 
	= : \lambda p^2 + c_{\beta} \Big(d e^{\theta} p^{- c_T / (4d)} \Big)^{\beta}.
	\end{align} 
	By the definition of $g_{Z_{\ell t}}^*$ and minimizing both side of \eqref{ResultOne} with respect to $p$, we obtain 
	\begin{align} \label{ApproximationEstimationSingle}
	\lambda p_{Z_{\ell t}}^{*2}
	+ \mathcal{R}_{L, \mathrm{P}}(g_{Z_{\ell t}}^*) - \mathcal{R}_{L, \mathrm{P}}^*
	\le c_{d,\beta} \ e^{\frac{8 d \theta \beta}{c_{T} \beta + 8 d}} \lambda^{\frac{c_T \beta}{c_T \beta + 8 d}}
	\end{align}
	with the constant $c_{d, \beta} := (c_T \beta + 8 d)
	(c_T \beta)^{- \frac{c_T \beta}{c_T \beta + 8 d}}
	(c_\beta d^\beta/(8d))^{\frac{8 d}{c_T \beta + 8 d}}$.
	Here, the minimum is achieved when taking
	\begin{align*}
	p = \Big(\frac{ c_T c_{\beta}\beta d^\beta e^{\theta \beta}}{8d\lambda}\Big)^{\frac{4d}{c_T\beta+8d}}.
	\end{align*}
	
	From the definition of $g_{Z_t}^*$ in \eqref{BestScoreMinimizerPopulation},
	we see that it is the decision function that minimize the approximation error among the $k$ independent trials $g_{Z_{\ell t}}^*$, $\ell = 1, \ldots, k$. 
	Therefore, the independence together with \eqref{ApproximationEstimationSingle} implies that
	\begin{align*}
	& \mathrm{P}_Z \bigg( \lambda p_{Z_t}^{*2} + \mathcal{R}_{L, \mathrm{P}} (g_{Z_t}^*) - \mathcal{R}_{L, \mathrm{P}}^*
	\geq c_{d, \beta} \ e^{\frac{8 d \beta \theta}{c_{T} \beta + 8 d}} \lambda^{\frac{c_T \beta}{c_T \beta + 8 d}} \bigg)
	\\
	& = \mathrm{P}_Z \bigg( \forall \ell \in \{1, \ldots, k\},\ \lambda p_{Z_{\ell t}}^{*2} + \mathcal{R}_{L, \mathrm{P}}(g_{Z_{\ell t}}^*) - \mathcal{R}_{L, \mathrm{P}}^*
	\geq c_{d, \beta} \ e^{\frac{8 d \beta \theta}{c_{T} \beta + 8 d}} \lambda^{\frac{c_T \beta}{c_T \beta + 8d}} \bigg)
	\\
	& = \bigg( \mathrm{P}_Z \bigg( \lambda p_{Z_{\ell t}}^{*2} + \mathcal{R}_{L, \mathrm{P}}(g_{Z_{\ell t}}^*) - \mathcal{R}_{L, \mathrm{P}}^*
	\geq c_{d, \beta} \ e^{\frac{8 d \beta \theta}{c_{T} \beta + 8 d}} \lambda^{\frac{c_T \beta}{c_T \beta + 8 d}}\bigg)\bigg)^k
	\le e^{- k \theta}.
	\end{align*}
	Simple calculation with $\tau := k \theta$
	yields that 
	\begin{align*}
	\lambda p_{Z_t}^{*2} + \mathcal{R}_{L, \mathrm{P}}(g_{Z_t}^*) - \mathcal{R}_{L, \mathrm{P}}^*
	\le c_{d, \beta} \ 
	e^{\frac{8 d \tau \beta}{k(c_T \beta + 8 d)}}
	\lambda^{\frac{c_T \beta}{c_T \beta + 8 d}}
	\end{align*}
	holds with probability $\mathrm{P}_Z$ at least $1-e^{-\tau}$.
\end{proof}

\begin{proof}[Proof of Lemma \ref{VCdimension}]
	The proof will be conducted by means of
	geometric constructions. 
	
	\begin{figure*}[htbp]
		\centering
		\begin{minipage}[b]{0.18\textwidth}
			\centering
			\includegraphics[width=\textwidth]{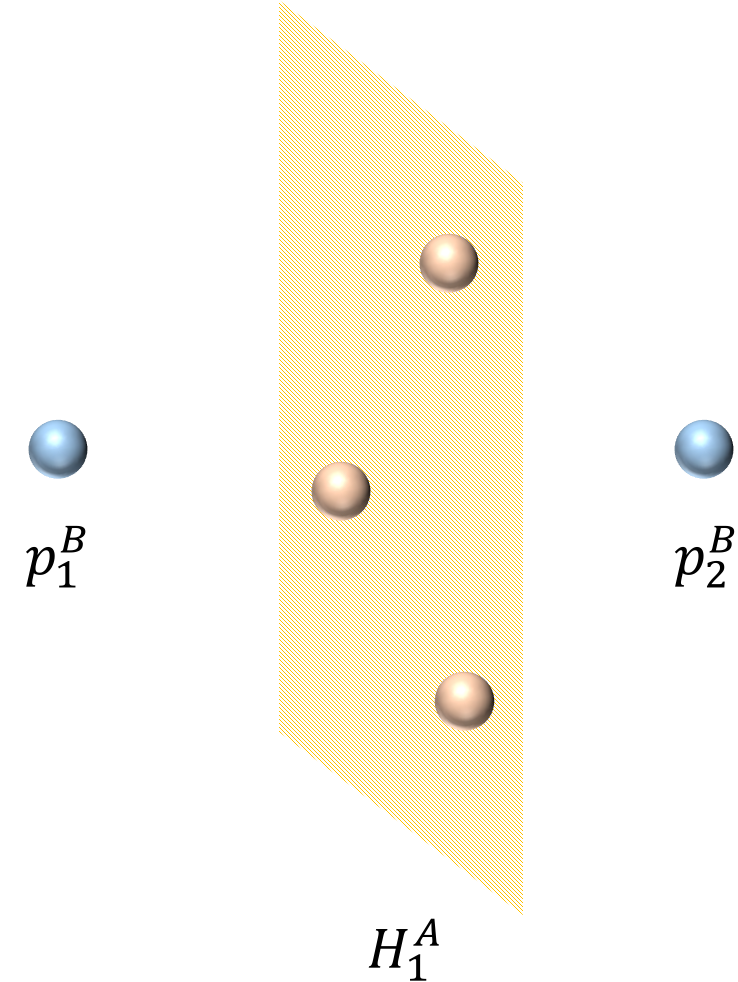}
			$p=1$
			\centering
			\label{fig::p=1}
		\end{minipage}
		\qquad
		\begin{minipage}[b]{0.25\textwidth}
			\centering
			\includegraphics[width=\textwidth]{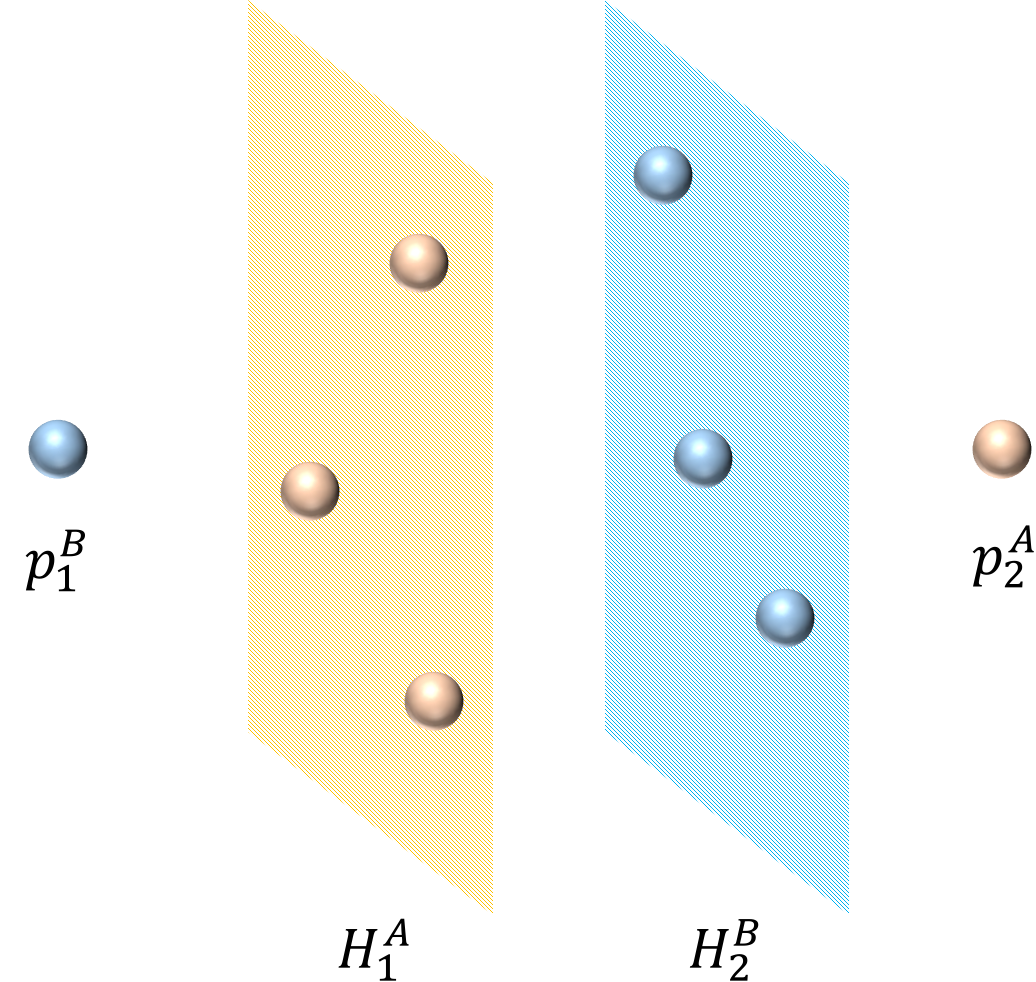}
			$p=2$
			\label{fig::p=2}
		\end{minipage}
		\qquad
		\begin{minipage}[b]{0.43\textwidth}
			\centering
			\includegraphics[width=\textwidth]{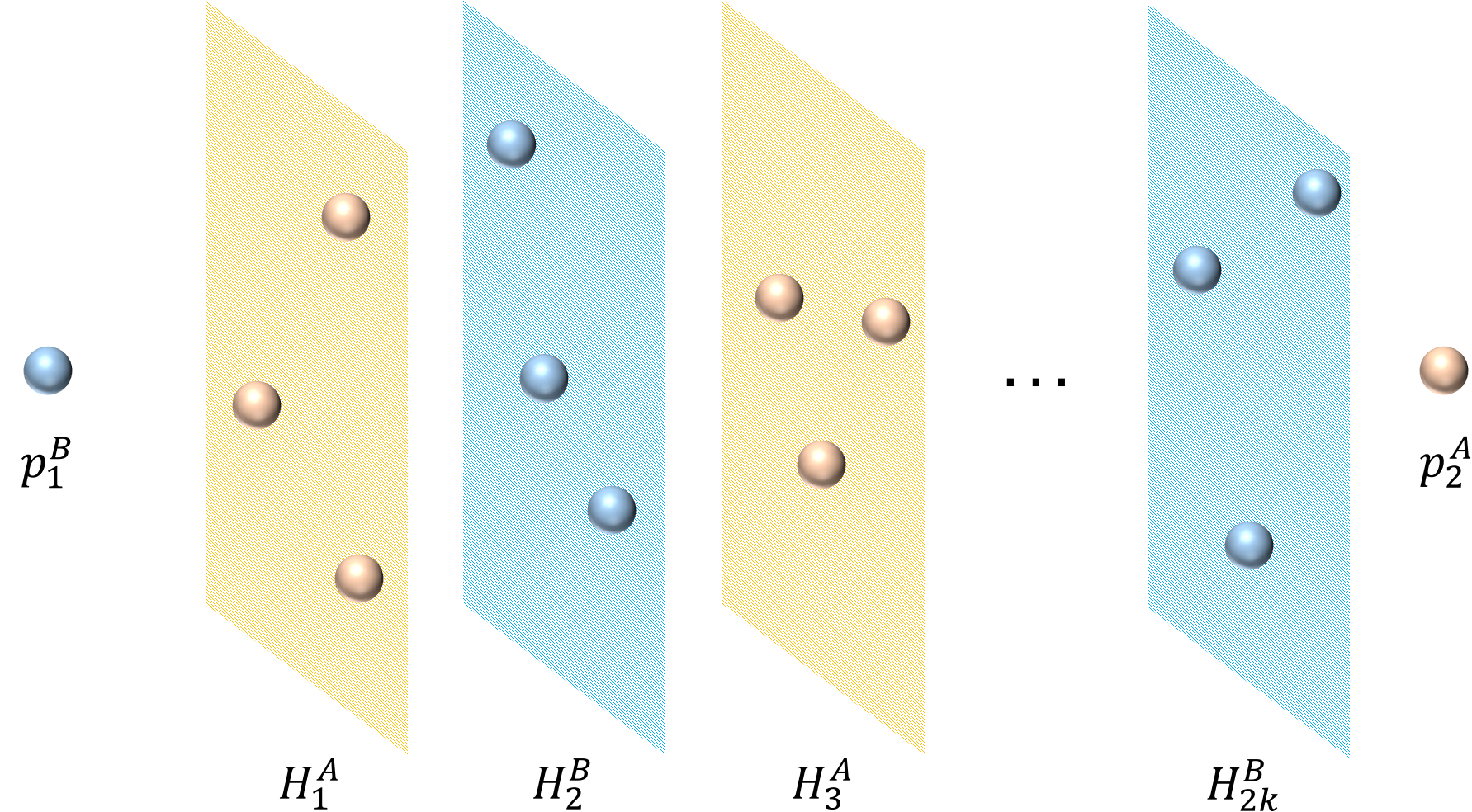}
			$p=2k$
			\label{fig::p=2k}
		\end{minipage}
		\caption{We take one case with $d=3$ as an example to illustrate the geometric interpretation of the VC dimension. The yellow balls represent samples from class $A$, blue ones are from class $B$ and slices denote the hyperplanes formed by samples. }
		\label{fig::VC}
	\end{figure*}

	Let us start with the observation of a partition with number of splits $p=1$. Since the dimension of the feature space is $d$, the smallest number of points that cannot be divided by $p=1$ split is $d+2$. To be specific, due to the fact that $d$ points can be used to form $d-1$ independent vectors and thus a hyperplane of a $d$-dimensional space, we consider the case where there is a hyperplane consisting of $d$ points all from one class labeled as $A$, and two points from the other class $B$ located on different sides of the hyperplane, respectively. That hyperplane is denoted as $H_1^A$ for the sake of convenience. Here, points from two classes cannot be separated by one split, that is, one hyperplane,  which means that $\mathrm{VC}(\mathcal{B}_1) \leq d + 2$. 
	
	We next consider the partition with number of splits $p=2$ by extending the above case. If we pick one point out of the two that are located on either side of the above hyperplane $H_1^A$, then we can add $d-1$ more points from class $B$ so that they together form a hyperplane $H_2^B$ parallel to $H_1^A$. Then, we add a new point from class $A$ to the side of the newly constructed hyperplane $H_2^B$. Here, the newly added point should be on the opposite side to $H_1^A$. In this case, $p=2$ splits cannot separate these points of the total number $2d+2$ from two different classes. Consequently, we conclude that $\mathrm{VC}(\mathcal{B}_2) \leq 2d + 2$.
	
	Inductively, the above analysis can be extended to the general case of number of splits $p \in \mathbb{N}$. What we have to do is to add points continuously to form $p$ (mutually) parallel hyperplanes, and any two adjacent hyperplanes should be constructed from different classes. W.l.o.g.~we may consider the case for $p=2k+1$, $k \in \mathbb{N}$, where two points (denoted as $p_1^B$, $p_2^B$) from class $B$ and $2k+1$ alternately appearing hyperplanes should be located as $p_1^B, H_1^A, H_2^B, H_3^A, H_4^B, \ldots, H_{(2k+1)}^A, p_2^B$. 
	By this construction, we see that the smallest number of points that cannot be divided by $p$ splits is $dp+2$, which leads to $\mathrm{VC}(\mathcal{B}_p) \leq d p + 2$.
	
	It should be noted that our hyperplanes can be generated vertically which are in line with our splitting criteria for the decision trees. This completes the proof.
\end{proof}

\begin{proof}[Proof of Lemma \ref{BpTpCoveringNumbers}]
	The inequality \eqref{BpCoveringNumber} follows directly from 
	Lemma \ref{VCdimension} and Theorem 9.2 in \cite{Kosorok08}.
	For the inequality \eqref{TpCoveringNumber}, denote the covering number of $\eins_{\mathcal{B}_p}$ with respect to $\| \cdot \|_{L_2(Q)}$ as $\mathcal{N}(\varepsilon) : = \mathcal{N} \bigl( \eins_{\mathcal{B}_p}, \|\cdot\|_{L_2(Q)}, \varepsilon \bigr)$. Then, there exist $\eins_{B_1}, \ldots, \eins_{B_{\mathcal{N}(\varepsilon)}} \in \eins_{\mathcal{B}_p}$ such that the function set $\{\eins_{B_1}, \ldots, \eins_{B_{\mathcal{N}(\varepsilon)}}\}$ is an $\varepsilon$-net of $\eins_{\mathcal{B}_p}$ with respect to $\|\cdot\|_{L_2(Q)}$. This implies that for any $\eins_{B} \in \eins_{\mathcal{B}_p}$, there exists a $j \in \{ 1, \ldots, \mathcal{N}(\varepsilon) \}$ such that $\| \eins_B - \eins_{B_j} \|_{L_2(Q)} \leq \varepsilon$. 
	Now, for all $g \in \mathcal{T}_p$, the equivalent definition \eqref{Tp2} of $\mathcal{T}_p$ tells us that there exists a $\eins_{B} \in \eins_{\mathcal{B}_p}$
	such that $g$ can be written as $g = \eins_B - \eins_{B^c} = 2 \eins_B - 1$. The above discussion yields that there exists a $j \in \{ 1, \ldots, \mathcal{N}(\varepsilon) \}$ such that for $g_j := 2 \eins_{B_j} - 1$, there holds
	\begin{align*}
	\| g - g_j \|_{L_2(Q)} 
	= \| (2 \eins_B - 1) - (2 \eins_{B_j} - 1) \|_{L_2(Q)} 
	= \| 2 \eins_B - 2 \eins_{B_j} \|_{L_2(Q)} 
	= 2 \| \eins_B - \eins_{B_j} \|_{L_2(Q)} 
	\leq 2 \varepsilon.
	\end{align*}
	This implies that $\{ g_1, \ldots, g_{\mathcal{N}(\varepsilon)} \}$ is a $2 \varepsilon$-net of $\mathcal{T}_p$ with respect to $\|\cdot\|_{L_2(Q)}$. Consequently, we obtain
	\begin{align*}
	\mathcal{N} \bigl( \mathcal{T}_p, \|\cdot\|_{L_2(Q)}, \varepsilon \bigr) 
	\leq \mathcal{N} \bigl( \eins_{\mathcal{B}_p}, \|\cdot\|_{L_2(Q)}, \varepsilon / 2 \bigr) 
	\leq K (d p + 2) (4 e)^{d p + 2} ( 2 / \varepsilon )^{2 (d p + 1)},
	\end{align*}
	we thus proved the assertion.
\end{proof}

\begin{proof}[Proof of Lemma \ref{HrEntropyNumber}]
	First, we notice that
	for all $g \in \mathcal{G}_r$, the number of splits $p$ must satisfy $p \leq q := \lfloor (r/\lambda)^{1/2} \rfloor$. 
	Then, the nested relation \eqref{NestRelation} implies that $\mathcal{G}_r \subset \mathcal{T}_q$. 
	Therefore, Lemma \ref{BpTpCoveringNumbers} implies that the covering number of $\mathcal{G}_r$ with respect to $L_2(D)$ satisfies
	\begin{align} \label{GrCoveringNumber}
	\mathcal{N} ( \mathcal{G}_r, \|\cdot\|_{L_2(D)}, \varepsilon)
	\leq \mathcal{N} ( \mathcal{T}_q, \|\cdot\|_{L_2(D)}, \varepsilon)
	\leq  K (d q + 2) (4e)^{d q + 2} ( 2 / \varepsilon)^{2 ( d q + 1)}.
	\end{align}
	For the classification loss $L$, we have for any $h \in \mathcal{H}_r$, 
	\begin{align*}
	h(x, y) = L \circ g \, (x, y) - L \circ f_{L, \mathrm{P}}^*\, (x, y) = \eins_{\{ g(x) \neq y \}} - \eins_{\{ f_{L, \mathrm{P}}^*(x) \neq y \}},
	\end{align*}
	where $g \in \mathcal{G}_r$. This implies that for any $h, \tilde{h} \in \mathcal{H}_r$, there exist $g, \tilde{g} \in \mathcal{G}_r$
	\begin{align*}
	\| h - \tilde{h} \|_{L_2(D)} 
	& = \biggl( \frac{1}{n} \sum_{i=1}^n \bigl( h(x_i, y_i) - \tilde{h}(x_i, y_i) \bigr)^2 \biggr)^{1/2} 
	    = \biggl( \frac{1}{n} \sum_{i=1}^n \bigl( \eins_{\{ g(x) \neq y \}} - \eins_{\{ \tilde{g} (x) \neq y \}} \bigr)^2 \biggr)^{1/2}
	\\
	& = \frac{1}{2}  \biggl( \frac{1}{n} \sum_{i=1}^n \bigl( g(x) - \tilde{g} (x) \bigr)^2 \biggr)^{1/2}
	    = \frac{1}{2} \| g - \tilde{g} \|_{L_2(D)}. 
	\end{align*}
	Therefore, similarly as the proof in Lemma \ref{BpTpCoveringNumbers}, we obtain
	\begin{align*}
	\mathcal{N} (\mathcal{H}_r, \|\cdot\|_{L_2(D)}, \varepsilon)
	\leq \mathcal{N} (\mathcal{G}_r, \|\cdot\|_{L_2(D)}, 2 \varepsilon)
	\leq K (d q + 2) (4e)^{d q + 2} ( 1 / \varepsilon)^{2 ( d q + 1)},
	\end{align*}
	where the later inequality follows from the estimate \eqref{GrCoveringNumber}. Elementary calculations show that for any $ 0< \varepsilon < 1/ \max \{ e, K \}$ and $q \geq 2$, there holds
	\begin{align*}
	\log \mathcal{N} (\mathcal{H}_r, \|\cdot\|_{L_2(D)}, \varepsilon)
	\leq 12 d q \log ( 1 / \varepsilon )
	\leq 12 d (r / \lambda)^{1/2} \log ( 1 / \varepsilon ).
	\end{align*}
	Consequently, for all $\delta \in (0, 1)$, we have
	\begin{align} \label{Differenting}
	\sup_{\varepsilon \in (0, 1 / \max \{ e, K \})} \  \varepsilon^{2 \delta} \log \mathcal{N} (\mathcal{H}_r, \|\cdot\|_{L_2(D)}, \varepsilon)
	\leq 12 d (r / \lambda)^{1/2} \sup_{\varepsilon \in (0,1)} \varepsilon^{2 \delta} \log ( 1 / \varepsilon ).
	\end{align}
	For any fixed $\delta \in (0, 1)$,
	simple analysis shows that the right hand side of \eqref{Differenting} is maximized
	at $\varepsilon^* = e^{- 1 / (2 \delta)}$ and consequently we obtain
	\begin{align*}
	\log \mathcal{N} (\mathcal{H}_r, \|\cdot\|_{L_2(D)}, \varepsilon)
	\le \frac{6 d}{e \delta} \Big(\frac{r}{\lambda}\Big)^{1/2} \varepsilon^{-2 \delta}.
	\end{align*}
	Then, Exercise 6.8 in \cite{StCh08} implies that the entropy number bound of $\mathcal{H}_r$ with respect to $L_2(D)$ satisfies
	\begin{align*}
	e_i (\mathcal{H}_r, \|\cdot\|_{L_2(D)})
	\leq \biggl( \frac{18 d}{e \delta} \biggr)^{1/(2 \delta)}
	\biggl( \frac{r}{\lambda} \biggr)^{1/(4 \delta)}
	i^{- 1/(2 \delta)}.
	\end{align*}
	Obviously, this bound holds for $\mathbb{E}_{D \sim \mathrm{P}^n} \, e_i (\mathcal{H}_r, \|\cdot\|_{L_2(D)})$ as well. The proof is finished.
\end{proof}

\begin{proof}[Proof of Proposition \ref{Rademacher}]
	First of all, Lemma \ref{VarianceBoundLemma} shows that 
	for all $h \in \mathcal{H}_r$, there holds
	$
	\mathbb{E}_{\mathrm{P}} h^2
	\leq V r^\vartheta
	=: \sigma^2.
	$
	Now, $\|h\|_{\infty} \leq 1 =: B$ and
	$a := c_{d, \delta} (r/\lambda)^{1/(4\delta)} \geq B$ in Lemma \ref{HrEntropyNumber} together with Theorem 7.16 in \cite{StCh08}
	yields that
	\begin{align*}
	\mathbb{E}_{D \sim \mathrm{P}^n} \ 
	\mathrm{Rad}_D (\mathcal{H}_r, n)
	\leq \max \Bigl\{ c_{\delta,1} c_{d, \delta}^\delta V^{\frac{1 - \delta}{2}} \lambda^{- \frac{1}{4}} r^{\frac{2 \vartheta - 2 \vartheta \delta + 1}{4}} n^{- \frac{1}{2}}, c_{\delta,2} c_{d, \delta}^{\frac{2 \delta}{1 + \delta}} \lambda^{- \frac{1}{2 + 2 \delta}} r^{\frac{1}{2 + 2 \delta}} n^{- \frac{1}{1 + \delta}}
	\Bigr\},
	\end{align*}
	where 
	$c_{\delta,1} : = \frac{2 \sqrt{\log 256} c_\delta^\delta}{(\sqrt{2} - 1) (1 - \delta) 2^{\delta / 2}}$
	and
	$c_{\delta,2} : = \Big(\frac{8 \sqrt{\log 16} c_\delta^\delta}{(\sqrt{2} - 1) (1 - \delta) 4^{\delta}} \Big)^{\frac{2}{1 + \delta}}$,
	with
    $c_\delta := \frac{\sqrt{2} - 1}{\sqrt{2} - 2^{\frac{2 \delta - 1}{2 \delta}}}
	\cdot \frac{1 - \delta}{\delta}$,
	which can be upper bounded by $46e$ and $1035 e^2$ respectively.
\end{proof}

To prove Theorem \ref{OracleInequality}, we still need the following lemma which gives the variance bound for the classification loss.

\begin{lemma} \label{VarianceBoundLemma}
	Let $L$ be the classification loss,
	$\mathcal{T}$ be the function set defined as in \eqref{Space}.
	Then, for all $g \in \mathcal{T}$, the variance bound
	\begin{align} \label{VarianceBound}
	\mathbb{E} \bigl( L \circ g - L \circ f_{L, P}^* \bigr)^2 
	\leq V \bigl( \mathbb{E} ( L \circ g - L \circ f_{L, P}^* ) \bigr)^\vartheta
	\end{align}
	holds for $\vartheta = \alpha / (1 + \alpha)$ and the constant $V$ depending on $\alpha$.
\end{lemma}

\begin{proof}[Proof of Lemma \ref{VarianceBoundLemma}]
	For the classification loss $L$, there holds
	\begin{align*}
	\big( L \circ g - L \circ f_{L, P}^* \big)^2 
	= \big( \eins_{\{ g(x) \neq y \}} - \eins_{\{ f_{L, P}^*(x) \neq y \}} \big)^2
	= \frac{1}{2} \big| g(x) - f_{L, P}^*(x) \big|.
	\end{align*}
	Moreover, Example 3.8 in \cite{StCh08} implies that
	\begin{align*}
	\mathbb{E} \big( L \circ g - L \circ f_{L, P}^* \big) 
	= \frac{1}{2} \int_{\mathcal{X}} \big| 2 \eta(x) - 1 \big| 
	\big| g(x) - f_{L, P}^*(x) \big| \, d\mathrm{P}_X(x)
	\end{align*}
	holds for $L$ as well.
	Combining the above two equalities and using Assumption \ref{NoiseExponent}, we obtain 
	\begin{align} 
	& \mathbb{E} ( L \circ g - L \circ f_{L, P}^*)^2
	\nonumber\\
	& = \frac{1}{2} \mathbb{E} \bigl| g(x) - f_{L, P}^*(x) \bigr|
	= \frac{1}{2} \int_\mathcal{X} \bigl| g(x) - f_{L, P}^*(x) \bigr| \ d \mathrm{P}_X(x)
	\nonumber\\
	& = \frac{1}{2} \int_\mathcal{X} \bigl| g(x) - f_{L, P}^*(x) \bigr| \eins_{\{|2\eta(x) - 1| \leq h\}} \ d \mathrm{P}_X(x)
	+ \frac{1}{2} \int_\mathcal{X} \bigl| g(x) - f_{L, P}^*(x) \bigr| \eins_{\{|2\eta(x) - 1| > h\}} \ d \mathrm{P}_X(x)
	\nonumber\\
	& \leq \mathrm{P}_X \bigl( \{ x \in \mathcal{X} : |2 \eta(x) - 1| \leq h \} \bigr) 
	+ \frac{1}{2h} \int_\mathcal{X} \bigl| 2 \eta(x) - 1 \bigr|  \bigl| g(x) - f_{L, P}^*(x) \bigr| \ d \mathrm{P}_X(x)
	\nonumber\\
	& \leq c_\alpha h^\alpha + h^{-1} \mathbb{E} ( L \circ g - L \circ f_{L, P}^*)
	\label{VarianceBoundInequality}
	\end{align}
	for all $h > 0$. Minimizing \eqref{VarianceBoundInequality} with respect to $h$, we get
	\begin{align*}
	\mathbb{E} \bigl( L \circ g - L \circ f_{L, P}^* \bigr)^2
	\leq c \bigl( \mathbb{E} ( L \circ g - L \circ f_{L, P}^* ) \bigr)^{\alpha / (1 + \alpha)}
	\leq (1 + c) \bigl( \mathbb{E} ( L \circ g - L \circ f_{L, P}^* ) \bigr)^{\alpha / (1 + \alpha)}
	\end{align*}
	with $c : = c_\alpha^{1 / (1 + \alpha)} \bigl( \alpha^{1 / (1 + \alpha)} + \alpha^{- \alpha / (1 + \alpha)} \bigr)$. The assertion is thus proved.
\end{proof}

\begin{proof}[Proof of Theorem \ref{OracleInequality}]
	Obviously, for the classification loss $L$, we have $\| L \|_{\infty} \leq 1$. Therefore,  
	the supreme bound is satisfied for $B = 1$ and the constant $B_0$ can be taken as $1$ as well. Moreover, Lemma \ref{VarianceBoundLemma} implies that the variance bound \eqref{VarianceBound} holds for $V$ and $\vartheta$. The condition that $V \ge B^{2 - \vartheta}$ is satisfied.
	
	Let us first consider the case $\lambda n^2 < 1$. Obviously, we have
	\begin{align*}
	\lambda p_{Z_t}^2 + \mathcal{R}_{L, \mathrm{P}}(g_{Z_t}) - \mathcal{R}_{L, \mathrm{P}}^*
	\leq \lambda p_{Z_t}^2 + \mathcal{R}_{L, D}(g_{Z_t}) + B
	\leq \mathcal{R}_{L, D} (1) + B
    \leq 2 B \big(\lambda^{-1} n^{-2} \big)^{\frac{1}{3 - 2 (1 - \delta) \vartheta}}.
	\end{align*}
	For the case $\lambda n^2 \geq 1$, let $c_{\delta, 1}$, $c_{\delta, 2}$ be the constants defined as in Proposition \ref{Rademacher},
	$c_{d, \delta}$ be the constant as in Lemma \ref{HrEntropyNumber}.
	For fixed $n \geq 1$ and all $r > r^*$, 
	Proposition \ref{Rademacher} implies that the expectations of the empirical Radmacher averages of $\mathcal{H}_r$ \eqref{Hr}
	can be upper bounded by
	the function $\varphi_n(r)$ defined by
	\begin{align*}
	\varphi_n(r)
	: = \max \Bigl\{ c_{\delta, 1} c_{d, \delta}^\delta V^{\frac{1 - \delta}{2}} \lambda^{- \frac{1}{4}} r^{\frac{2 (1 - \delta) \vartheta + 1}{4}} n^{- \frac{1}{2}},
	c_{\delta, 2} c_{d, \delta}^{\frac{2 \delta}{1 + \delta}} \lambda^{- \frac{1}{2 (1 + \delta)}} r^{\frac{1}{2 (1 + \delta)}} n^{- \frac{1}{1 + \delta}} \Bigr\}.
	\end{align*}
	It can be easily concluded that, for this $\varphi_n$,
	the condition $\varphi_n(4r) \leq c \varphi_n(r)$ is satisfied
	if $c \geq 2 \sqrt{2}$. This implies that the statements of the Peeling Theorem 7.7 in \cite{StCh08} still hold for $\varphi_n(4r) \leq 2 \sqrt{2} \varphi_n(r)$. In other words,
	the constant $2$ in front of $\varphi_n(r)$ in that theorem should be replaced by $2 \sqrt{2}$
	for our case. Accordingly, the assumption concerning  $\varphi_n$ and $r$ in Theorem 7.20 in \cite{StCh08} should be modified to $\varphi_n(4r) \leq 2 \sqrt{2} \varphi_n(r)$ and
	\begin{align*}
	r > \max \biggl\{ 75 \varphi_n(r), 
	\biggl( \frac{72 V \tau}{n} \biggr)^{\frac{1}{2-\vartheta}}, 
	\frac{5 B_0 \tau}{n}, r^* \biggr\}, 
	\end{align*}
	respectively.
	Now, using $(1 - \delta)(1 - \vartheta) > 0$, some elementary calculations show that the condition $r \geq 75 \varphi_n(r)$ is satisfied if
	\begin{align*}
	r \geq C \lambda^{-\frac{1}{3 - 2 (1 - \delta) \vartheta}} 
	n^{- \frac{2}{3 - 2 (1 - \delta) \vartheta}},
	\end{align*}
	where
	\begin{align*}
	C = \max \biggl\{ \Bigl( 75 c_{\delta, 1} c_{d, \delta}^\delta V^{\frac{1 - \delta}{2}} \Bigr)^\frac{4}{3 - 2 (1 - \delta) \vartheta}, 
	\Bigl( 75 c_{\delta, 2} c_{d, \delta}^{\frac{2 \delta}{1 + \delta}} \Bigr)^{\frac{2 (1 + \delta)}{1 + 2 \delta}}
	\biggr\}.
	\end{align*}
	Finally, from the definition \eqref{rstar}, we obviously have $r^* \leq \lambda p_{Z_t}^{*2} + \mathcal{R}_{L, P}(g_{Z_t}^*) - \mathcal{R}_{L, \mathrm{P}}^*$. The assertion follows from Theorem 7.20 in \cite{StCh08} by
	taking $K = \max \{2B, 3C\}$.
\end{proof}

\begin{proof}[Proof of Theorem \ref{ConvergenceRates}]
	Theorem \ref{OracleInequality} together with	
	Proposition \ref{ApproxError} implies that with probability $\mathrm{P} \otimes \mathrm{P}_Z$ at least $1 - 4 e^{- \tau}$, there holds
	\begin{align} \label{OracleApprox}
	& \lambda p_{Z_t}^2 + \mathcal{R}_{L,\mathrm{P}}(g_{Z_t}) - \mathcal{R}_{L,\mathrm{P}}^*
	\nonumber\\
	& \leq 9 c_{d, \beta} \ e^{\frac{8 d \beta \tau}{k (c_{T} \beta + 8 d)}} \lambda^{\frac{c_T \beta}{c_T \beta + 8 d}}+ K \bigl( \lambda^{-1} n^{-2} \bigr)^{\frac{1}{3 - 2 (1 - \delta) \vartheta}} + 3 \bigl( 72 V \tau n^{-1} \bigr)^{\frac{1}{2 - \vartheta}} + 15 \tau n^{-1},
	\end{align}
	where $c_{d,\beta}$ is the constant as in Proposition \ref{ApproxError}, 
	$K$ and $V$ are the constants as in Theorem \ref{OracleInequality}.
	Now, minimizing the first two terms in the right hand side of \eqref{OracleApprox}
	with respect to $\lambda$, we obtain
	\begin{align*}
	9 c_{d,\beta} \ e^{\frac{8 d \beta \tau}{k (c_T \beta + 8 d)}}
	\lambda^{\frac{c_T \beta}{c_T \beta + 8 d}}
	+ K \bigl( \lambda^{-1} n^{-2} \bigr)^{\frac{1}{3 - 2 (1 - \delta) \vartheta}}
	\leq c_{d,\vartheta,\beta,\delta} \  c_{d,\vartheta,\beta,\delta,\tau}^{1/k} \  n^{-\frac{c_T\beta}{c_T \beta (2 - (1 - \delta) \vartheta) + 4 d}},  
	\end{align*}
	with the constant $c_{d, \vartheta, \beta, \delta, \tau}$ and $c_{d,\vartheta,\beta,\delta}$ defined by
	\begin{align*}
	& c_{d,\vartheta,\beta,\delta,\tau} 
	:= e^{{\frac{4 \beta d \tau}{c_T \beta (2 - (1 - \delta) \vartheta) + 4d}}}> 1, \\
	& c_{d,\vartheta,\beta,\delta} := \frac{2K (c_T \beta (2 - (1 - \delta) \vartheta) + 4 d)}{c_T \beta (3 - 2 (1 - \delta) \vartheta)}
	\biggl( \frac{9 c_T c_{d,\beta} \beta (3 - 2 (1 - \delta) \vartheta)}{K (c_T \beta + 8 d)} \biggr)^{\frac{c_T \beta + 8 d}{2 (c_T \beta (2 - (1 - \delta) \vartheta) + 4 d)}}.
	\end{align*}
	Here, the minimum is attained when choosing
	\begin{align*}
	\lambda = 
	\biggl( \frac{K (8d/(c_T c_\beta \beta d^\beta))^{\frac{8 d}{c_T \beta + 8 d}} }
	{9(3 - 2 (1 - \delta) \vartheta)} \biggr)^{\frac{(c_T \beta + 8 d)(3 - 2 (1 - \delta) \vartheta)}{2(c_T \beta (2 - (1 - \delta) \vartheta) + 4 d)}}
	e^{- \frac{4 \beta d \tau (3 - 2 (1 - \delta) \vartheta)}{k (c_T \beta (2 - (1 - \delta) \vartheta) + 4 d)}}
	n^{-\frac{c_T\beta+8d}{c_T \beta (2 - (1 - \delta) \vartheta) + 4 d}}.
	\end{align*}
	Now, for all $n > 1$,
	the inequality \eqref{OracleApprox} becomes
	\begin{align}
	\lambda p_{Z_t}^2 & + \mathcal{R}_{L,\mathrm{P}}(g_{Z_t}) - \mathcal{R}_{L,\mathrm{P}}^*
	\nonumber\\
	& \leq c_{d,\vartheta,\beta,\delta} \  c_{d,\vartheta,\beta,\delta,\tau}^{1/k} \ 
	n^{-\frac{c_T\beta}{c_T \beta (2 - (1 - \delta) \vartheta) + 4 d}}
	+3 \big(72 V \tau n^{-1} \big)^{\frac{1}{2 - \vartheta}} 
	+ 15 \tau n^{-1}
	\label{Estimation}
	\\
	& \leq C n^{- \frac{c_T \beta}{c_T \beta (2 - (1 - \delta) \vartheta) + 4 d}},
	\nonumber
	\end{align}
	with
	$C : = c_{d, \vartheta, \beta, \delta} \ c_{d, \vartheta, \beta, \delta, \tau}^{1/k} + 3 (72 V \tau)^{1/(2-\vartheta)} + 15 \tau$.
\end{proof}

\begin{proof}[Proof of Theorem \ref{the::OptimalConvergenceRates}]
Similar as the proof of Theorem \ref{ConvergenceRates}, under the Assumptions \ref{MarginNoiseExponent} and \ref{ass::DistanceControlsNoise}, Lemma 8.23 in \cite{StCh08} implies that $\mathrm{P}$ has margin exponent $q := \gamma \alpha$ and margin-noise exponent $\beta := \gamma(\alpha + 1)$. 
\end{proof}

\begin{proof}[Proof of Theorem \ref{ConvergenceRateForest}]
	For the classification loss $L$, there holds 
	\begin{align*}
	\mathcal{R}_{L, \mathrm{P}} (f) - \mathcal{R}_{L, \mathrm{P}}^* = \frac{1}{2} \int_\mathcal{X} \big|2 \eta(x) - 1 \big| \big|f(x) - f_{L, \mathrm{P}}^*(x) \big|\ d \mathrm{P}_X(x),
	\end{align*}
	see e.g. Example 3.8 in \cite{StCh08}. 
	Using the definition \eqref{vPlusvMinus} and Markov's inequality, we obtain
	\begin{align*}
	& \mathcal{R}_{L,\mathrm{P}}(f_Z) - \mathcal{R}_{L, \mathrm{P}}^*
	\\
	& = \frac{1}{2} \int |2 \eta(x) - 1| |f_Z - f_{L, \mathrm{P}}^*| \ d \mathrm{P}_X(x)
	\\
	& = \int \big|2 \eta(x) - 1 \big| \big| \eins_{\{v_{+}(x) \geq \frac{m}{2}\}} 
	- \eins_{\{ \eta(x) \geq \frac{1}{2}\}} \big| \ d \mathrm{P}_X(x)
	\\
	& = \int |2 \eta(x) - 1| \eins_{\{v_{+}(x) \geq \frac{m}{2}\}} 
	\eins_{\{\eta(x) < \frac{1}{2}\}} 
	+ |2 \eta(x) - 1| \eins_{\{v_{-}(x) > \frac{m}{2}\}} 
	\eins_{\{\eta(x) \geq \frac{1}{2}\}}  \ d \mathrm{P}_X(x)
	\\
	& \leq \frac{2}{m} \int |2 \eta(x) - 1| v_{+}(x) \eins_{\{\eta(x) < \frac{1}{2} \}}
	+ |2 \eta(x) - 1| v_{-}(x) \eins_{\{\eta(x) \geq \frac{1}{2}\}} \ d \mathrm{P}_X(x)
	\\
	& = \frac{2}{m} \sum_{t=1}^m 
	\int |2 \eta(x) - 1| \eins_{\{g_{Z_t} = 1\}} 
	\eins_{\{\eta(x) < \frac{1}{2}\}} + |2 \eta(x) - 1| \eins_{\{g_{Z_t} = -1\}} 
	\eins_{\{\eta(x) \geq \frac{1}{2}\}} \ d \mathrm{P}_X(x)
	\\
	& = \frac{2}{m} \sum_{t=1}^m 
	\int \big|2 \eta(x) - 1 \big| \big| \eins_{\{g_{Z_t}(x) = 1\}} 
	- \eins_{\{\eta(x) \geq \frac{1}{2}\}} \big| \ d \mathrm{P}_X(x)
	\\
	& = \frac{2}{m} \sum_{t=1}^m 
	\bigl( \mathcal{R}_{L, \mathrm{P}}(g_{Z_t}) - \mathcal{R}_{L, \mathrm{P}}^* \bigr).
	\end{align*}
	Now, denote the estimation in \eqref{Estimation} by $E$, then the union bound yields that, for all $\tau > 0$, there holds
	\begin{align*}
	\mathrm{P} \otimes \mathrm{P}_Z
	\Bigl( \mathcal{R}_{L,\mathrm{P}}(f_Z) - \mathcal{R}_{L,\mathrm{P}}^* > 2 E
	\Bigr)
	\leq \sum_{t=1}^m \mathrm{P} \otimes \mathrm{P}_Z
	\Bigl( \mathcal{R}_{L,\mathrm{P}}(g_{Z_t}) - \mathcal{R}_{L,\mathrm{P}}^* 
	>  E \Bigr)
	\leq 4 m e^{-\tau}.
	\end{align*}
	Consequently, 
	with probability $\mathrm{P} \otimes \mathrm{P}_Z$ at least $1 - 4 e^{- \tau}$, there holds
	\begin{align}
	& \mathcal{R}_{L,\mathrm{P}}(f_Z) - \mathcal{R}_{L,\mathrm{P}}^*
	\nonumber\\
	& \le 2 c_{d,\vartheta,\beta,\delta} \ c_{m, d, \vartheta, \beta, \delta, \tau}^{1/k} \ 
	n^{-\frac{c_T\beta}{c_T \beta (2 - (1 - \delta) \vartheta) + 4 d}}
	+6 \big(72 V (\tau + \log m) n^{-1} \big)^{\frac{1}{2 - \vartheta}} 
	+ 30 (\tau + \log m) n^{-1}
	\nonumber\\
	& \leq C n^{- \frac{c_T \beta}{c_T \beta (2 - (1 - \delta) \vartheta) + 4 d} },
	\nonumber
	\end{align}
	where the constants $c_{m, d, \vartheta, \beta, \delta, \tau} := e^{{\frac{4 \beta d (\tau + \log m)}{c_T \beta (2 - (1 - \delta) \vartheta) + 4d}}} > 1$ and $C :=  2 c_{d,\vartheta,\beta,\delta} \ c_{m, d, \vartheta, \beta, \delta, \tau}^{1/k} + 6 (72 V$ $(\tau + \log m))^{\frac{1}{2-\vartheta}} + 30(\tau + \log m)$. This proves the assertion.
\end{proof}

\begin{proof}[Proof of Theorem \ref{the::ForestOptimalConvergenceRates}]
	Similar as the proof of Theorem \ref{ConvergenceRates}, it can be proved that
	if Assumptions \ref{MarginNoiseExponent} and \ref{ass::DistanceControlsNoise} hold, then Lemma 8.23 in \cite{StCh08} implies that $\mathrm{P}$ has margin exponent $q:= \gamma \alpha$ and margin-noise exponent $\beta := \gamma(\alpha + 1)$. 
\end{proof}

\section{Conclusion}\label{sec::conclusion}
In the present paper, we propose an algorithm named \emph{best-scored random forest} for binary classification problems. The terminology \emph{best-scored} means to select one tree with the best empirical performance out of certain number of purely random tree candidates. 
From the theoretical point of view, learning theory analysis is conducted within the framework of regularized empirical risk minimization with penalty on the number of splits. Here, we achieve almost optimal convergence rates under proper assumptions on margin conditions. 
Furthermore, a counterexample is presented to verify that all dimensions must have the probability to be split in order to ensure the consistency of the purely random tree. 
Concerning with numerical experiments,
we improve the original purely random splitting criterion to an adaptive one
leading to the efficiency of the number of splits.
Finally, to further verify the effectiveness of our algorithm, comparisons between best-scored random forest and other state-of-art classification methods, such as Breiman's random forest, 
extremely randomized trees,
$k$-nearest neighbors, and 
support vector machines 
are also conducted.

\bibliographystyle{chicago}

\end{document}